\documentclass{article}

\usepackage[preprint]{neurips_2026}

\usepackage{amsmath}
\usepackage{multirow}
\usepackage{mathrsfs}
\usepackage[utf8]{inputenc}
\usepackage[ruled,linesnumbered]{algorithm2e}

\usepackage{subcaption}

\usepackage[utf8]{inputenc} 
\usepackage[T1]{fontenc}    
\usepackage[
  colorlinks=true,
  linkcolor=cyan,
  citecolor=cyan
]{hyperref}       
\usepackage{url}            
\usepackage{booktabs}       
\usepackage{adjustbox}      
\usepackage{array}
\usepackage{amsfonts}       
\usepackage{nicefrac}       
\usepackage{microtype}      
\usepackage{xcolor}         

\usepackage{booktabs}

\usepackage{microtype}
\usepackage{graphicx}
\usepackage{amssymb}
\usepackage{multicol}
\usepackage{wasysym}
\usepackage{makecell}
\usepackage{pifont}
\usepackage{booktabs} 
\usepackage{float}
\usepackage{hyperref}

\definecolor{commentcolor}{RGB}{220, 50, 47} 

\usepackage[bb=dsserif]{mathalpha}
\usepackage{bm}

\usepackage[utf8]{inputenc}
\usepackage{amsmath}
\usepackage{amssymb}
\usepackage{mathtools}
\usepackage{amsthm}
\usepackage{hyperref} 
\usepackage{wrapfig}
\usepackage{colortbl}
\newcommand{\cmark}{\ding{51}} 
\newcommand{\xmark}{\ding{55}} 

\usepackage[capitalize,noabbrev]{cleveref}

\theoremstyle{plain}
\newtheorem{theorem}{Theorem}[section]

\newtheorem{lemma}[theorem]{Lemma}

\theoremstyle{definition}
\newtheorem{definition}[theorem]{Definition}

\theoremstyle{remark}
\newtheorem{remark}[theorem]{Remark}
\usepackage{amssymb}

\usepackage{mdframed}

\newmdtheoremenv[
  backgroundcolor=white!95!blue!5,  
  linecolor=blue!40!cyan,          
  linewidth=1pt,
  roundcorner=6pt,
  innertopmargin=8pt,
  innerbottommargin=8pt,
  innerrightmargin=10pt,
  innerleftmargin=10pt,
  frametitlebackgroundcolor=cyan!10, 
  frametitleaboveskip=2pt,
  frametitlebelowskip=2pt,
  frametitlefont=\bfseries,
  shadow=true,
  shadowsize=3pt,
  shadowcolor=cyan!15,
]{theoremwithframe}{Theorem}

\newenvironment{talign*}
{\csname align*\endcsname}
{\endalign}
\newenvironment{talign}
{\align}
{\endalign}

\title{Fisher Decorator: Refining Flow Policy via \\  a Local Transport Map}

\author{
Xiaoyuan Cheng$^{1}$\thanks{Equal contribution. Contact Xiaoyuan Cheng with \url{ucesxc4@ucl.ac.uk}, and Haoyu Wang with \url{create_arc@sjtu.edu.cn}.} \quad
Haoyu Wang$^{2}$\footnotemark[1] \quad 
Wenxuan Yuan$^{6}$ \quad 
Ziyan Wang$^{3}$ \\ \textbf{Zonghao Chen}$^{1}$  \quad \textbf{Li Zeng}$^{4}$ \quad 
\textbf{Zhuo Sun}$^{5,6}$\thanks{Correspondence Author. Correspondence to Zhuo Sun: \url{zhuosunreid@outlook.com}.} \\
$^{1}$University College London, 
$^{2}$Shanghai Jiao Tong University,
$^{3}$King's College London,\\ 
$^{4}$Peking University,
$^{5}$Shanghai University of Finance and Economics, \\
$^{6}$Imperial College London \\
}

\begin{document}

\maketitle

\begin{abstract}
Recent advances in flow-based offline reinforcement learning (RL) have achieved strong performance by parameterizing policies via flow matching. However, they still face critical trade-offs among expressiveness, optimality, and efficiency. In particular, existing flow policies interpret the $L_2$ regularization as an upper bound of the 2-Wasserstein distance ($W_2$), which can be problematic in offline settings. This issue stems from a fundamental geometric mismatch: the behavioral policy manifold is inherently anisotropic, whereas the $L_2$ (or upper bound of $W_2$) regularization is isotropic and density-insensitive, leading to systematically misaligned optimization directions. To address this, we revisit offline RL from a geometric perspective and show that policy refinement can be formulated as a local transport map—an initial flow policy augmented by a residual displacement. By analyzing the induced density transformation, we derive a local quadratic approximation of the KL-constrained objective governed by the Fisher information matrix, enabling a tractable anisotropic optimization formulation. By leveraging the score function embedded in the flow velocity, we obtain a corresponding quadratic constraint for efficient optimization. Our results reveal that the optimality gap in prior methods arises from their isotropic approximation. In contrast, our framework achieves a controllable approximation error within a provable neighborhood of the optimal solution. Extensive experiments demonstrate state-of-the-art performance across diverse offline RL benchmarks\footnote{See project page: \url{https://github.com/ARC0127/Fisher-Decorator}.}. 
\end{abstract}

\section{Introduction}
Offline reinforcement learning (RL) algorithms hold tremendous promise for transforming large datasets into powerful decision-making systems \citep{levine2020offline, fu2022closer}. A central challenge in offline RL is formulated as \textit{KL-constrained policy optimization}, which balances reward maximization with the constraint to remain near the behavioral distribution, thereby mitigating distributional shift \citep{prudencio2023survey}. As datasets grow larger and span increasingly diverse domains, their behavioral policy distributions become more complex and often highly multimodal \citep{o2024open, park2024ogbench}. Effectively modeling such distributions requires expressive policy parameterizations. Consequently, offline RL introduces a key challenge for policy learning: balancing policy expressiveness with reliable policy optimization.

Improving policy expressiveness is therefore crucial for solving complex decision-making problems. Recent works have explored modern generative modeling techniques, such as diffusion models \citep{hansen2023idql, chen2023score, janner2022planning} and flow-based models \citep{zhang2025energy}, to address this challenge. However, directly incorporating these models into policy extraction remains non-trivial. Existing approaches typically derive a generative policy from a learned value function using techniques such as weighted behavioral cloning \citep{lu2023contrastive, kang2023efficient, ding2024diffusion} or reparameterized policy gradients \citep{wang2022diffusion, ding2023consistency}. In practice, weighted behavioral cloning often suffers from a limited number of effective training samples and insufficient policy expressiveness, which can lead to suboptimal performance, particularly on complex tasks \citep{park2024value}. Meanwhile, reparameterized policy gradient methods require backpropagation through time \citep{espinosa2025expressive}, making them computationally expensive and prone to unstable optimization.

To improve training efficiency and stability, recent works have explored one-step distillation techniques that compress iterative generative processes into single-step generators \citep{salimans2022progressive, geng2025mean}. Building on this idea, recent state-of-the-art methods such as flow $Q$-learning \citep{park2025flowqlearning} and its variants \citep{zhang2026reform, nguyen2026onestepflowqlearningaddressing} distill iterative flow models into one-step generators, enabling more tractable policy optimization while remaining faithful to the behavioral distribution.
However, the distillation objective corresponds to minimizing an upper bound of the 2-Wasserstein distance \citep{villani2009optimal}. This is because existing methods typically resort to an $L_2$ relaxation as a computationally tractable surrogate for the $W_2$ objective \citep{park2025flowqlearning}. However, such a simplification inherently reduces the optimization to an isotropic, density-insensitive penalty. In other words, it treats all parts of the behavior policy’s support equally, regardless of how much probability mass the behavior policy assigns to them. As a result, it fails to distinguish between deviations in high-density regions (i.e., frequently visited state–action pairs under the behavior policy) and low-density regions. In contrast, KL constraints explicitly emphasize alignment in high-density regions, and this mismatch can lead to suboptimal updates and distributional shift. Consequently, the resulting one-step policy may suffer from mode collapse and face challenges such as distributional shift, especially when the $Q$ function is inaccurate. Another recent work \citep{mu2026deflow} introduces a residual refinement component with an $L_2$ penalty to mitigate these effects. However, this regularization still corresponds to the same $W_2$ upper bound, and thus inherits the same isotropic, density-insensitive behavior, leading to suboptimal policy updates.
This exposes a fundamental mismatch: offline RL relies on KL-constrained updates that are inherently anisotropic and density-aware, whereas existing flow-based methods rely on $W_2$ regularization whose practical instantiations reduce to isotropic regularization (see more comprehensive literature review in Appendix~\ref{append:literature_review}).

\vspace{-10pt}
\begin{figure}[h]
    \centering
    \includegraphics[width=0.9\linewidth]{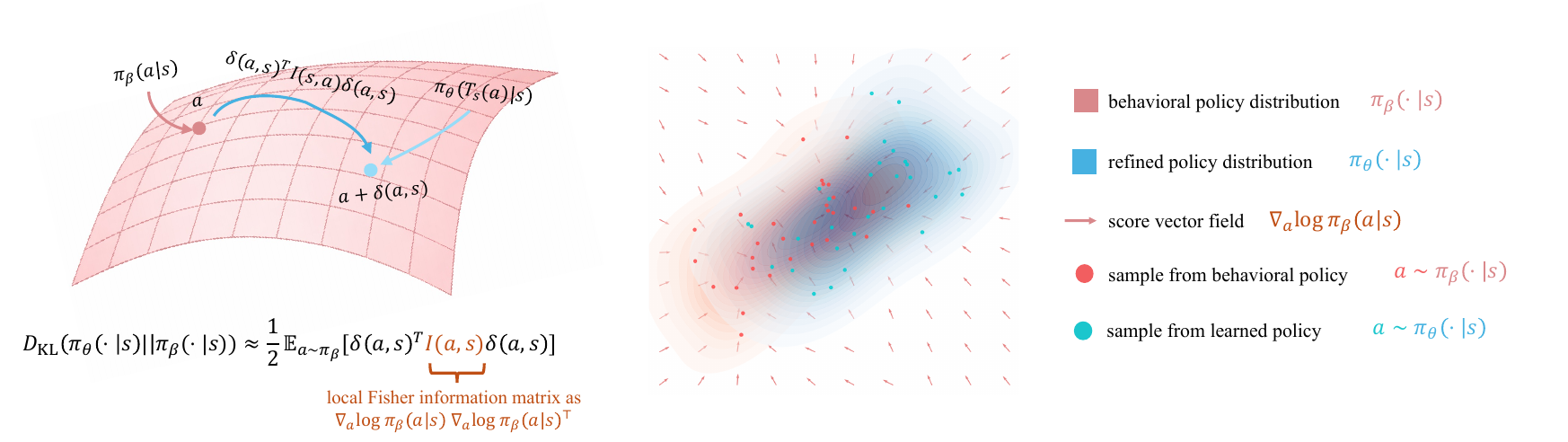}
    \caption{Geometric interpretation of offline policy optimization. (Left) Each action $a$ within the behavioral policy support is refined via a local transport map that induces a constrained displacement. The KL divergence between the behavioral policy $\pi_\beta(\cdot|s)$ and the refined pushforward measure $\pi_\theta(\cdot|s) = (T_s)_{\#} \pi_\beta(\cdot|s)$ is captured by the expectation of the local quadratic form, which is structurally governed by the anisotropic Fisher information matrix. (Right) Visualization of the policy evolution: samples from the behavioral policy distribution (red) are transported to the refined policy distribution (blue). Their trajectories are anchored by the Fisher information matrix, ensuring the refinement respects the underlying manifold.}
    \label{fig:concept}
\end{figure}
\vspace{-5pt}


To address this fundamental mismatch, we revisit flow-based policy learning in offline RL from a geometric perspective (see Figure~\ref{fig:concept}). We model policy refinement as a local transport map that augments the flow policy with a residual displacement. By analyzing the pushforward of the probability density through this transport map, we show that the KL constraint admits a local quadratic approximation governed by the Fisher information matrix. This reveals that policy updates inherently evolve on a statistical manifold with anisotropic structure \citep{shen2025nonparametric}, whose metric can be efficiently estimated from the flow’s velocity field. As a result, we obtain geometry-aware, anisotropic updates without requiring backpropagation through time or relying on biased one-step distillation. Based on this insight, we propose a simple residual policy, termed \textbf{Fisher Decorator (FiDec)}. Our contributions are summarized as follows:

(1) \textit{Geometric Characterization of the KL--Wasserstein Gap:}
We identify a fundamental geometric mismatch between KL-constrained policy optimization and its $W_2$ surrogate. By analyzing the pushforward induced by a local transport map, we show that the KL constraint admits a local quadratic form governed by the local Fisher information metric, which is inherently anisotropic and density-aware. In contrast, the $W_2$ upper bound reduces to an isotropic, density-insensitive $L_2$ penalty. This mismatch explains the optimality gap of prior methods and reveals why they induce distributional shift and mode collapse in multimodal settings.

(2) \textit{Tractable Anisotropic Policy Optimization:}
Building on this characterization, we reformulate KL-constrained policy optimization as a quadratic objective defined on the transport map. We show that the local Fisher information metric is naturally encoded in the score function of the velocity field and can be efficiently estimated. This yields a tractable anisotropic, metric-aware update rule that better preserves distribution support and avoids mode averaging.

(3) \textit{The Fisher Decorator (FiDec) Framework:}
We propose FiDec, a simple residual policy that parametrizes refinement as a local transport map. By constraining the residual under the local Fisher information metric, FiDec performs anisotropic, trust-region updates that are aligned with the underlying geometry. This mitigates distributional shift and mode collapse, while remaining efficient without ODE solving or one-step distillation. Empirically, FiDec achieves state-of-the-art (SOTA) performance across a wide range of benchmarks.

\section{Preliminary}
\textbf{Notation for Offline RL.}
We consider the standard Markov decision process (MDP) $\mathcal{M}$ \citep{sutton1998reinforcement}, defined by the tuple $(\mathcal{S}, \mathcal{A}, r, \rho, p)$, where $\mathcal{S} \subset \mathbb{R}^n$ is the state space, $\mathcal{A} \subset \mathbb{R}^d$ is the action space, $r: \mathcal{S} \times \mathcal{A} \rightarrow \mathbb{R}$ is the reward function, $\rho \in \Delta(\mathcal{S})$ is the initial state distribution, and $p: \mathcal{S} \times \mathcal{A} \rightarrow \Delta(\mathcal{S})$ denotes the transition dynamics.  In offline RL, the goal is to learn the parameters $\theta$ of a policy $\pi_\theta: \mathcal{S} \rightarrow \Delta(\mathcal{A})$ that maximizes the expected discounted return $\mathbb{E}_{s \sim \mathcal{D},\, a \sim \pi_\theta}[Q_{\phi}(s,a)] = \mathbb{E}_{s \sim \mathcal{D},\, a \sim \pi_\theta}[\sum_{t=0}^{N} \gamma^t r(s_t,a_t)]$ using a fixed dataset $\mathcal{D} = \{\tau^{(n)}\}_{n=1}^M$ without additional environment interaction. Here, each trajectory $\tau$ is given by $(s_0, a_0, \dots, s_N, a_N)$, and $\gamma \in (0,1)$ denotes the discount factor.

\textbf{Problem Formulation.}
Let $\pi_\beta$ denote the behavioral policy extracted from the dataset $\mathcal{D}$. Following \cite{levine2020offline} and \cite{prudencio2023survey}, offline RL can be formulated as a KL-constrained optimization problem:
\begin{talign} \label{eq:KL_constrained_problem}
\max_{\pi_\theta} \quad 
& \mathbb{E}_{s\sim \mathcal{D},\, a\sim \pi_\theta(\cdot|s)} 
[Q_{\phi}(s,a)], \qquad 
\text{s.t.}\quad  \mathbb{E}_{s\sim \mathcal{D}}
\left[
D_{\text{KL}}\big(\pi_\theta(\cdot|s)\,\|\,\pi_\beta(\cdot|s)\big)
\right] \le \epsilon ,
\end{talign}
where $\epsilon > 0$ is a predefined threshold that controls the trade-off between policy improvement and stability. This constraint ensures the learned policy $\pi_\theta$ remains within a principled proximity to $\pi_\beta$, thereby mitigating distributional shift by anchoring the optimization to the data support. In this work, we aim to learn an expressive policy $\pi_\theta$ that satisfies this objective. Specifically, we leverage the flow matching framework for policy parameterization, which facilitates the flexible modeling of complex, multimodal action distributions \citep{tong2023conditional}. To provide the necessary background, we next briefly review the formulation of flow matching and its associated velocity fields.

\textbf{Flow Matching and Policy.} Flow models \citep{lipman2022flow, liu2022flow, albergo2022building} are rooted in ordinary differential equations (ODEs), which enable simpler training and fast inference. Given a data distribution $p(x) \in \Delta(\mathbb{R}^d)$, flow matching aims to fit the parameter $\beta$ of a time-dependent velocity field $v_\beta:[0,1] \times \mathbb{R}^d \to \mathbb{R}^d$ that its corresponding flow \citep{lang2012differential} $\psi_\beta: [0,1] \times \mathbb{R}^d \to \mathbb{R}^d$, defined by the unique solution to the ODE
\begin{talign}
    \frac{d}{dt} \psi_\beta(t, x) = v_\beta(t, \psi_\beta(t, x)),
\end{talign}
transforms a unit Gaussian at $t = 0$ to the target distribution $p$ at $t = 1$. We consider the linear Gaussian interpolation path \citep{lipman2024flow}, the objective is to minimize
\begin{equation}
    \mathcal{L}_{\text{flow}} = \mathbb{E}_{t \sim \text{Unit}[0,1], x^0 \sim \mathcal{N}(0, I_d), x^1\sim p(x)} [ \| v_\beta(t, x^t) - (x^1 - x^0) \|_2^2 ], 
\end{equation}
where $\mathcal{N}(0, I_d)$ is the $d$-dimensional unit Gaussian, $\text{Unif}[0,1]$ denotes the uniform distribution over the unit interval, and $x^t = (1-t)x^0 + tx^1$ is the linear interpolation. 

\textbf{Flow Policies.} We use flow matching to learn the behavioral policy, and the objective is minimizing
\begin{equation}
    \mathbb{E}_{t \sim \text{Unit}[0,1], s, a = x^1  \sim \mathcal{D}, x^0 \sim  \mathcal{N}(0, I_d)} [ \| v_\beta (t, s, x^t) - (x^1 - x^0) \|_2^2],
\end{equation}
where $v_\beta: [0,1] \times \mathcal{S} \times \mathbb{R}^d \to \mathbb{R}^d$ is a state-time-dependent vector field that generates a state-dependent flow $\psi_\beta(t,s,x)$, which serves as the behavioral policy.  For simplicity, we denote $\mu_\beta(s, z)$ as the conditional flow policy at $t = 1$, where $z$ is sampled from the standard Gaussian $\mathcal{N}(0, I_d)$. Here, we ask the central question: 
\begin{center}
\textit{How can we design a policy refinement mechanism that respects the anisotropic structure induced by the KL constraint, while remaining tractable for optimization?}
\end{center}
To address this, the following section identifies the theoretical discrepancy between existing flow-based methods and the foundational optimization objective. We then examine the inherent link between the pushforward of transport maps and KL-constrained optimization.

\section{Methods}
This section first identifies the theoretical gap of previous methods, then introduces a local transport map to bridge KL-constrained optimization and information geometry, and finally derives a practical algorithm (FiDec) for efficient policy refinement. 

\subsection{Revisit Theoretical Gap in Optimization} \label{subsec:gap}
In Flow $Q$-learning (FQL) \citep{park2025flowqlearning}, its variants, as well as DeFlow \citep{mu2026deflow}, the refined policy is typically parameterized as a one-step flow $\mu_\theta(s, z)$. Under this formulation, the constraint term is equivalent to the upper bound of $2$-Wasserstein distance:
\begin{equation} \label{eq:W_2_distance}
\begin{split}
        \mathbb{E}_{s \sim \mathcal{D}, z\sim \mathcal{N}(0, I_d)} \big[\| \mu_\theta(s, z) - \mu_\beta(s, z)\|_2^2 \big] & \geq  \mathbb{E}_{s \sim \mathcal{D}} \big[\inf_{\xi \in \Xi(\pi_\theta(\cdot|s), \pi_\beta(\cdot|s))} \mathbb{E}_{x,y \sim \xi} [\| x - y \|_2^2] \big]
        \\
        &= \mathbb{E}_{s \sim \mathcal{D}} \big[W_2\big(\pi_\theta(\cdot|s),\pi_\beta(\cdot|s)\big)^2\big], 
\end{split}
\end{equation}
where $\Xi(\pi_\theta(\cdot|s), \pi_\beta(\cdot|s))$ denotes the set of coupling distributions of $\pi_\theta$ and  $\pi_\beta$, and $W_2$ denotes the $2$-Wasserstein distance
with the Euclidean metric in the action space. However,  the $W_2$ regularization in \eqref{eq:W_2_distance} is not equivalent to the KL constraint specified in \eqref{eq:KL_constrained_problem}.


\textbf{Why the $L_2$ (or upper bound of $W_2$) in \eqref{eq:W_2_distance} is problematic?}  
First, the nature of $D_{\text{KL}}$ is anisotropic \citep{amari2016information}, which emphasizes the alignment of probability densities through its logarithmic ratio, acting as a strict support-preserving constraint. Specifically, the KL divergence becomes infinite when $\pi_\theta(\cdot | s)$ assigns positive probability mass outside the support of $\pi_\beta(\cdot|s)$, effectively preventing the learned policy from drifting into regions outside the behavioral support. This property is crucial in offline RL, where inaccurate $Q$-values often exert an `extrapolative' pressure on the policy \citep{park2024value, zhang2026reform}. 

In contrast, replacing KL with $L_2$ (or upper bound of $W_2$) introduces fundamental limitations due to its isotropic nature, particularly in multimodal settings. Specifically, when instantiated via its commonly used upper bound, the $W_2$ objective reduces to an isotropic $L_2$ penalty in the ambient space, treating all directions and regions uniformly regardless of the underlying distribution. This leads to several undesirable consequences: (1) \textit{Relaxation of Support Constraints:} Unlike the `zero-forcing' nature of KL, the isotropic $L_2$ penalty lacks a mechanism to strongly penalize deviations into low-density or unsupported regions. As a result, the optimized policy may assign non-negligible mass to regions outside the behavioral support, leading to distributional shift and OOD issues. (2) \textit{Mass Averaging via Displacement Interpolation:} From the perspective of optimal transport \citep{chewi2025statistical, villani2009optimal}, $W_2$ optimization follows displacement geodesics. For a multimodal $\pi_\beta$, the isotropic penalty induces a `centripetal' effect toward the Wasserstein barycenter, encouraging the policy to collapse toward the geometric mean of distinct modes and resulting in unrealistic mode-averaging behavior. (3) \textit{Misaligned Updates under Intrinsic Geometry:} The isotropic $L_2$ regularization ignores the anisotropic statistical manifold induced by the behavioral distribution. Consequently, policy updates are misaligned with the underlying geometry, leading to suboptimal optimization.


To break this bottleneck, we revisit the policy refinement process and develop an optimization framework that is mathematically faithful to the original KL constraint while remaining computationally tractable within the flow-matching framework.

\subsection{Policy Refinement via A Local Transport Map}

\begin{figure}[ht]
    \centering
    \includegraphics[width=0.9\linewidth]{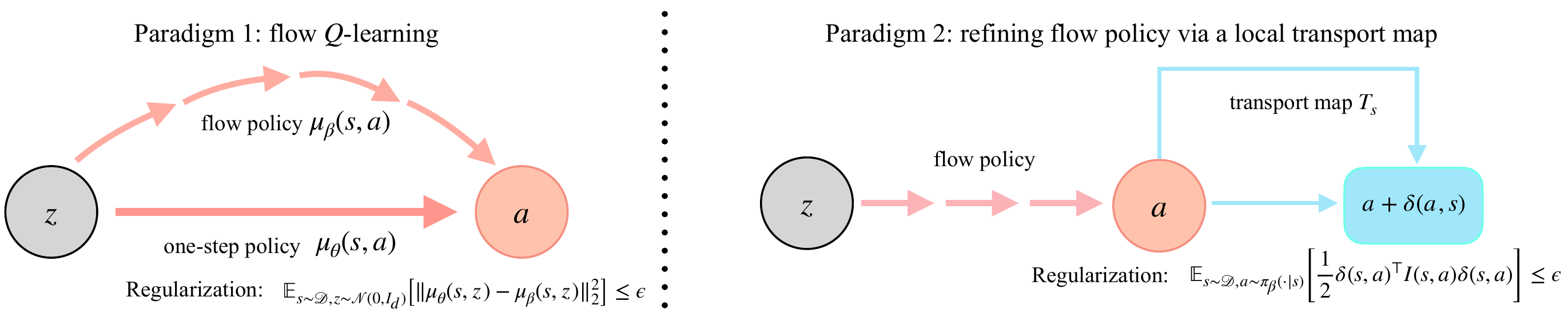}
    \caption{Comparison of flow policy refinement paradigms. (Left) Paradigm 1: flow $Q$-learning. The refined policy $\mu_\theta$ is parameterized as an independent one-step mapping from noise $z$ to action $a$. (Right) Paradigm 2: local transport map (Ours). Instead of re-learning the flow, we keep the original flow as a base and apply a local transport map $T_s$. The refined action is formulated as $T_s(a)=a +\delta(s,a)$, ensuring the refinement is constrained by local Fisher information matrix.}
    \label{fig:paradiam_comp}
    \vspace{-10pt}
\end{figure}

Rather than learning a one-step policy (see difference in Figure~\ref{fig:paradiam_comp}), we parameterize the refined policy $\pi_\theta$ via a differentiable local \textit{transport map} $T_s: \mathcal{A} \to \mathcal{A}$ for each $s \in \mathcal{S}$ (see Definition~\ref{def:transport_map} in Appendix~\ref{append:technical_definition}). This map is defined as a local displacement over the flow-generated action:
\begin{equation} \label{eq:local_transform}
    T_s(a) \coloneqq a + \delta(s, a) \quad \text{where} \ a = \mu_{\beta}(s,z). 
\end{equation}
Here, $\delta(s, \cdot): \mathbb{R}^d \to \mathbb{R}^d$ denotes a learned residual term that induces a local displacement field on the behavioral policy. By formulating the policy update as the \textit{pushforward measure}\footnote{Formally, the pushforward measure $\pi_\theta = (T)_{\#} \pi_\beta$ is defined such that for any measurable set $B \subset \mathcal{A}$, $\pi_\theta(B|s) = \pi_\beta(T_s^{-1}(B|s))$, meaning samples from $\pi_\theta$ are obtained by simply transforming samples from $\pi_\beta(\cdot|s)$ via $a' = T_s(a)$. More information can be found in Appendix~\ref{append:technical_definition}.} $\pi_\theta(\cdot|s) = (T_s)_{\#} \pi_\beta(\cdot|s)$, we can directly link the refinement process to the geometric structure of the behavioral policy distribution. This parameterization ensures that the optimization remains anchored to this support, inherently mitigating the risk listed in Section~\ref{subsec:gap}.

Applying the pushforward properties \citep{papamakarios2021normalizing}, 
the relationship between the refined policy distribution $\pi_\theta$ and the behavioral policy distribution $\pi_\beta$ is given by
\begin{equation} \label{eq:probability_pushforward}
\pi_{\theta}(T_s(a)| s) 
= 
\pi_\beta(a| s) 
\cdot 
\left|\det \nabla_{a'} T_s^{-1}(a')\right|_{a'=T_s(a)}.
\end{equation}
See detailed derivation in Appendix~\ref{lemma:change_of_variable_transport}. Here, $\det$ is the determinant of matrix and $\nabla_{a'} T^{-1}_s(a')$ denotes the Jacobian of the inverse map with respect to its input. Based on \eqref{eq:local_transform} and the inverse function theorem \citep{tao2006analysis}, the Jacobian of $T_s$ is
$$\nabla_{a'} T_s^{-1}(a') = (I_d + \nabla_a \delta(s, a))^{-1}$$
Since $\delta$ represents a small localized displacement field, we obtain
\begin{talign} \label{eq:det}
\left| \det \nabla_{a'} T_s^{-1}(a')  \right|
=
\left|\det (I_d + \nabla_a \delta(s,a))\right|^{-1}
=
1 - \nabla\cdot \delta(s,a) + \mathcal{O}(\|\nabla_a \delta(s,a)\|^2),
\end{talign}
where $\nabla \cdot \delta = \mathrm{tr}(\nabla_a \delta)$ denotes the divergence of the displacement field. The detailed derivation is in Appendix~\ref{append:technical_proof}. We then reveal how the local transport map changes the behavioral policy distribution.

\begin{theorem}[Connecting the Transport Map to KL Divergence] Let $\pi_\theta = T_{\#} \pi_\beta$ be the pushforward of the behavioral policy under the residual map $T(s,a) = a + \delta(s,a)$. Under a small perturbation $\delta$, the KL divergence admits the following second-order approximation:
\begin{equation}
D_{\text{KL}}\big(\pi_\theta(\cdot|s)\,\|\,\pi_\beta(\cdot|s)\big)   \approx \frac{1}{2} \mathbb{E}_{a \sim \pi_\beta(\cdot|s)} \left[ \delta(s,a)^\top \colorbox{cyan!15}{$I(s,a)$} \delta(s,a) \right],
\end{equation}
where \colorbox{cyan!15}{$I(s,a) = \nabla_a \log \pi_{\beta}(a|s) \nabla_a \log \pi_{\beta}(a|s)^\top$} is the local Fisher information matrix (see Definition~\ref{def:local_info_matrix}) obtained at the behavioral policy distribution.
\end{theorem}
Formal statements and proof are given in Appendix~\ref{theorem:second_order_expansion}. Notably, this establishes a natural bridge between the transport map and KL divergence without altering the original optimization objective. \textit{The final approximation reveals that the local displacement is penalized by an anisotropic quadratic form induced by the local Fisher information matrix $I(s,a)$.}


Having established this second-order approximation of KL divergence, ensuring consistency with \eqref{eq:KL_constrained_problem} reduces to estimating the local score function, $\nabla_a \log \pi_\beta(a|s)$. Remarkably, in the context of flow matching, this score function is intrinsically linked to the learned velocity field $v_\beta$. For a Gaussian interpolation path, the score function can be derived using Lemma~\ref{lemma:score_computation} in the Appendix as follows: 
\begin{talign} \label{eq:score_function_compute}
    \nabla_a \log \pi_\beta (a|s) = \lim_{t \to 1} \frac{t v_{\beta}(t, s, a) - a}{1-t}. 
\end{talign}
However, the expression in \eqref{eq:score_function_compute} exhibits a singularity in this limit (if directly calculate it via the velocity of flow), as both the numerator and denominator vanish simultaneously, leading to numerical instability. To circumvent this, we introduce a small perturbation $\varepsilon$ to the time variable, yielding a numerically stable approximation. The following theorem establishes the convergence rate and error bound between this perturbed approximation and the ground-truth score function.

\begin{theorem}[Pointwise Convergence of the Perturbed Score Function]    Let $\pi_\beta (a|s)$ denote the behavioral policy distribution defined on a compact set of $\mathbb{R}^d,$ and assume $\pi_\beta (\cdot|s) \in C^3(\mathbb{R}^d, \mathbb{R})$ over the compact subset of $\mathbb{R}^d$. Consider the linear Gaussian interpolation path, let $\mathbf{s}_t(a|s)$ be the associated score function at time $t$. For a small perturbation parameter $\varepsilon > 0$, define the perturbed time $t_\varepsilon = 1 - \varepsilon$. As $\varepsilon \to 0^+$, the approximation error between the perturbed score $\mathbf{s}_{1-\varepsilon}(a|s)$ and the ground truth score $\mathbf{s}_1(a|s)$ satisfies: 
    $$ \| \mathbf{s}_{1-\varepsilon}(a|s) - \mathbf{s}_1(a|s) \| = \varepsilon \left\| \nabla_a \big( a^\top \nabla_a \log \pi_\beta(a|s) \big) \right\| + \mathcal{O}(\varepsilon^2). $$ 
    Consequently, the boundary perturbation scheme achieves a first-order convergence rate of $\mathcal{O}(\varepsilon)$.
\end{theorem}
Detailed proof is given in Appendix~\ref{theorem:perturbation_error}. Based on the error analysis provided in the theorem above, we conclude that the Fisher information matrix $I(s,a)$ can be estimated via local perturbation with a controllable error bound. Thus, in flow policy, the optimization in \eqref{eq:KL_constrained_problem} can be surrogate as 
\begin{equation} \label{eq:final_formulation}
\begin{split}
    \max_{T} \quad 
& \mathbb{E}_{s\sim \mathcal{D},\, a\sim \pi_\beta(\cdot|s)} 
[Q_{\phi}(s,T_s(a))], \quad
\text{s.t.}\quad  \mathbb{E}_{s\sim \mathcal{D}, a\sim \pi_\beta(\cdot|s)}
\bigg[ \frac{1}{2} \delta(s, a)^\top I(s,a) \delta(s, a) \bigg] \le \epsilon .
\end{split}
\end{equation}
Here, the original intractable problem becomes a trust-region optimization \citep{nocedal2006numerical} as
\begin{talign} \label{eq:approximated_info_metric}
    I(s,a) \approx \frac{\big( t_\varepsilon v_{\beta}(t_\varepsilon, s, a) - a \big) \big( t_\varepsilon v_{\beta}(t_\varepsilon, s, a) - a \big)^{\top}}{(1 - t_\varepsilon)^2}
\end{talign}
is computed from the time perturbed score function.  The new formulation \eqref{eq:final_formulation} breaks previous bottlenecks (see the high-level comparison in Table~\ref{tab:comparison_regularization}):  
(1) \textit{Preservation of Support.} Unlike global $W_2$ optimization, the residual formulation enforces a local coupling where $\pi_\theta(a + \delta(s,a) | s) \approx \pi_\beta(a | s)$ due to the near-unit Jacobian $|\det \nabla_a T| \approx 1$ (since $\delta$ is small). This ensures that the probability mass is merely "nudged" within the existing manifold rather than being reallocated across space. By maintaining $\pi_\theta(a') \approx 0$ whenever $\pi_\beta(a|s) = 0$, the policy is effectively prevented from shifting into the "voids" between modes, thereby reinstating the support-preserving nature of the KL divergence. 
\begin{wraptable}{r}{0.7\linewidth}
\centering
\scriptsize
\setlength{\tabcolsep}{2.5pt}
\renewcommand{\arraystretch}{1.2}
\caption{Comparison of metric-aware behavioral regularizers in offline RL. Our method aligns with the KL-constrained formulation, preserving support and preventing mode averaging.}
\label{tab:comparison_regularization}
\begin{tabular}{lcccc}
\toprule
Method & Metric & Preserve Support & Robust to Mode Averaging & Anisotropic \\
\midrule
FQL  \citep{park2025flowqlearning} & $W_2^2$ & \xmark & \xmark & \xmark \\
DeFlow \citep{mu2026deflow} & $W_2^2$ & \xmark & \xmark & \xmark \\
FiDec (ours) & Fisher & \cmark & \cmark & \cmark \\
\bottomrule
\end{tabular}
\end{wraptable}
(2) \textit{Robust to Mode Averaging.} By restricting the update to a small displacement field $\delta$, the optimization is transformed from a global search for the Wasserstein barycenter into a localized geometric refinement. Since the transport map $T$ is anchored to near identity, it preserves the multimodal structure of $\pi_\beta$ and counteracts the "centripetal" force that typically leads to mass averaging. (3) \textit{Aligned Updates under Intrinsic Geometry.} Unlike the isotropic $L_2$ penalty induced by $W_2$, our formulation employs the local Fisher information metric to provide direction-aware regularization, better aligning policy updates with the structure of the behavioral distribution.

Furthermore, we analyze the value gap induced by the isotropic metric (i.e., $L_2$) from a trust-region optimization perspective. Under a first-order Taylor expansion, we show that this gap scales with the mismatch between the anisotropic metric and the $L_2$ metric as $\mathcal{O}\big(\nabla_a Q_\phi(s,a)^\top I(s,a)^{^{\dagger}} \nabla_a Q_\phi(s,a) - \nabla_a Q_\phi(s,a)^\top\nabla_a Q_\phi(s,a)\big)$ (see derivation in Appendix~\ref{append:optimal_policy_gap}). In general data distributions where the induced geometry is anisotropic, this gap remains non-vanishing.

\subsection{Practical Algorithm}
The Fisher Decorator (FiDec) algorithm (shown in Algorithm \ref{algo:FiDec} in Appendix~\ref{append:implementation}) is conceptually simple and can be decomposed into four components:  (1) training the critic $Q_{\phi}$, 
(2) learning the behavioral flow policy via the velocity field $v_\beta$, and 
(3) optimizing the transport map $T$, (4) updating the Lagrangian dual variables. 
Steps (1) and (2) largely follow the standard flow policy setting. In contrast, we emphasize the learning of the transport map $T$, and dual updates for the related Lagrangian formulation. 

\textbf{Learning Transport Map.}
We parameterize the transport map as a composition of a behavioral flow policy and a residual term. 
Specifically, the flow policy $\mu_\beta$ generates a base action, while a neural network $\delta_{\theta}$ learns a correction term to further refine the action. 
The residual term $\delta_{\theta}$ is constrained by a quadratic form induced by the local Fisher information matrix.

Following this formulation, we introduce the Lagrangian of \eqref{eq:final_formulation} as (see details in Appendix~\ref{append:trust_region}):
\begin{talign} \label{eq:primal-dual}
\min_{\lambda \geq 0} \max_{\delta_{\theta}} \;
\mathbb{E}_{s \sim \mathcal{D}, \, a \sim \pi_\beta}\Big[
    Q_\phi\big(s, T_s(\mu_\beta(s, z))\big) 
    - \lambda \big( \frac{1}{2} \delta_{\theta}(s,a)^\top I(s,a)\delta_{\theta}(s,a) - \epsilon \big)
\Big],
\end{talign}
where $T_s(\mu_\beta(s, z)) = \mu_\beta(s, z) + \delta_\theta(s, \mu_\beta(s, z))$ denotes the learned transport map, $I(s,a)$ is the approximated Fisher information matrix shown in \eqref{eq:approximated_info_metric}, $\lambda$ is the dual variable, and $\epsilon$ controls the strength of the constraint. 
Learning the transport map is equivalent to solving the inner problem. 

\textbf{Dual Updates.} We update the dual variable $\lambda$ via projected gradient ascent to enforce the constraint
\begin{talign}
\lambda \leftarrow \mathrm{ReLU}\!\left(
\lambda + \eta \, \mathbb{E}_{s \sim \mathcal{D}, \, a \sim \pi_\beta}\big[
\frac{1}{2}\delta_{\theta}(s,a)^\top I(s,a)\delta_{\theta}(s,a) - \epsilon
\big]
\right),
\end{talign}
where $\eta$ is the learning rate. This update increases $\lambda$ when the constraint is violated and decreases it otherwise, thereby adaptively balancing policy improvement and constraint satisfaction.

We optimize the objective~\eqref{eq:primal-dual} via an alternating primal-dual procedure.  In each iteration, we update the transport map parameter $\delta_\theta$ by maximizing the Lagrangian with fixed $\lambda$, and update $\lambda$ via projected gradient ascent.  In practice, expectations are approximated using minibatch samples, and gradients through the behavioral policy $\mu_\beta$ are stopped when optimizing $\delta_\theta$ for stability.

\section{Numerical Experiments}

\begin{wrapfigure}{r}{0.52\linewidth}
\centering
\vspace{-15pt}
\includegraphics[width=0.9\linewidth]{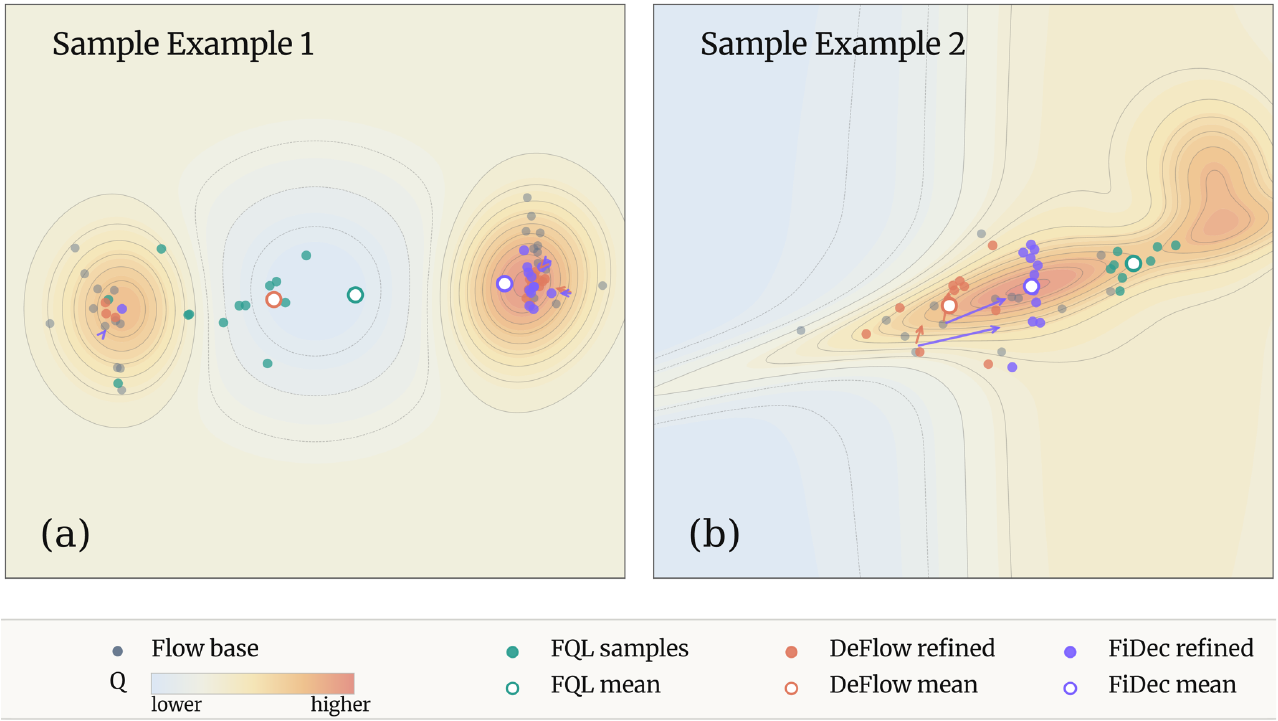}
\vspace{-1pt}
\caption{Isotropic vs. anisotropic policy refinement. 
(a) In multimodal settings, isotropic $W_2$-based methods exhibit mode averaging or interpolation: FQL collapses toward intermediate regions, while DeFlow spreads mass averagely across modes with its mean located in low-value areas. In contrast, FiDec preserves multimodality while shifting mass toward more favorable regions. 
(b) Along the distribution support, isotropic regularization either leads to OOD drift (FQL) or fails to reach high-value regions (DeFlow). FiDec instead follows the support and refines the policy toward higher-value regions.}
\label{fig:examples}
\vspace{-19pt}
\end{wrapfigure}

In this section, we benchmark our proposed algorithm, FiDec, against state-of-the-art offline and offline-to-online RL methods across a diverse suite of challenging tasks. We answer the questions if our method can break the existing bottleneck shown in Table~\ref{tab:comparison_regularization}. Furthermore, we conduct extensive ablation studies and sensitivity analyses to elucidate the impact of key hyperparameter choices.


\textbf{Benchmarks.} Followed previous work \citep{tarasov2023revisiting, rafailov2024d5rl, park2025flowqlearning}, we assess our method on multiple challenging benchmarks, including OGBench \citep{park2024ogbench} and D4RL \citep{fu2020d4rl}.

\textbf{Baselines for Offline Setting.} To ensure a comprehensive evaluation, we compare our method against eleven recent offline RL algorithms spanning a broad range of architectural designs and policy extraction paradigms. These include: (1) Gaussian policies: BC, IQL \citep{kostrikov2021offline}, and ReBRAC \citep{tarasov2023revisiting}; (2) diffusion policies: IDQL \citep{hansen2023idql}, SRPO \citep{chen2023score}, and CAC \citep{ding2023consistency}; and (3) flow policies: FAWAC \citep{ashvin2020accelerating}, IFQL \citep{zhang2025energy}, FQL \citep{park2025flowqlearning}, and DeFlow \citep{mu2026deflow}. \textbf{Baselines for Offline-to-Online Setting.} For offline-to-online RL experiments, we consider prior offline RL methods (IQL, ReBRAC, IFQL, FQL and DeFlow) that support fine-tuning and achieve strong performance. Meanwhile, we consider two other strong algorithms designed for online fine-tuning: Cal-QL \citep{nakamoto2023cal} and RLPD \citep{ball2023efficient}.

\textbf{Result Analysis.} Our result analysis is structured around
the core contributions of this work.

\textit{1. Can FiDec achieve strong performance in the offline setting across diverse benchmark tasks?} The answer is \textbf{YES}. As shown in Table~\ref{tab:offline_rl_latest_final}, FiDec consistently outperforms prior methods, including those based on Gaussian, diffusion, and flow policies. In particular, FiDec demonstrates clear advantages over recent strong baselines such as FQL and DeFlow, especially on challenging OGBench tasks (e.g., Humanoid, Puzzle, Antsoccer), which are inherently multimodal. This implicitly suggests that FiDec is able to effectively capture and preserve complex multimodal behavior distributions, while avoiding mode collapse and maintaining generalization across diverse task settings.

\begin{table*}[ht]
\centering
\caption{Offline RL results. Results are averaged over 8 seeds. Our method (FiDec) achieves strong performance across 73 diverse tasks in 4 benchmarks: the top three methods are highlighted in \colorbox{cyan!15}{(\textbf{1st})},  \colorbox{green!15}{(2nd)} and \colorbox{yellow!15}{(3rd)}. We follow the default evaluation setting from \cite{park2025flowqlearning} and \cite{mu2026deflow}. More detailed results for each specific task are reported in  Table~\ref{tab:full_offline_rl_latest_final}.}
\label{tab:offline_rl_latest_final}
\scriptsize
\setlength{\tabcolsep}{3.2pt}
\renewcommand{\arraystretch}{0.75}
\begin{adjustbox}{max width=\textwidth}
\begin{tabular}{>{\raggedright\arraybackslash}p{4.8cm}cccccccccccc}
\toprule
& \multicolumn{3}{c}{Gaussian Policies} & \multicolumn{3}{c}{Diffusion Policies} & \multicolumn{6}{c}{Flow Policies} \\
\cmidrule(lr){2-4} \cmidrule(lr){5-7} \cmidrule(lr){8-13}
Task Category & BC & IQL & ReBRAC & IDQL & SRPO & CAC & FAWAC & FBRAC & IFQL & FQL & DeFlow & \textbf{FiDec} \\
\midrule
OGBench antmaze-large-singletask (5) & 11$\pm$1 & 53$\pm$3 & \colorbox{yellow!15}{81$\pm$5} & 21$\pm$5 & 11$\pm$4 & 33$\pm$4 & 6$\pm$1 & 60$\pm$6 & 28$\pm$5 & 79$\pm$3 & \colorbox{green!15}{81$\pm$11} & \textbf{\colorbox{cyan!15}{87$\pm$1}} \\

OGBench antmaze-giant-singletask (5) & 0$\pm$0 & 4$\pm$1 & \colorbox{cyan!15}{\textbf{26$\pm$8}} & 0$\pm$0 & 0$\pm$0 & 0$\pm$0 & 0$\pm$0 & 4$\pm$4 & 3$\pm$2 & \colorbox{yellow!15}{9$\pm$6} & 5$\pm$7 & \colorbox{green!15}{13$\pm$4} \\

OGBench humanoidmaze-medium-singletask (5) & 2$\pm$1 & 33$\pm$2 & 22$\pm$8 & 1$\pm$0 & 1$\pm$1 & 53$\pm$8 & 19$\pm$1 & 38$\pm$5 & \colorbox{green!15}{60$\pm$14} & \colorbox{yellow!15}{58$\pm$5} & 57$\pm$29 & \colorbox{cyan!15}{\textbf{72$\pm$2}} \\

OGBench humanoidmaze-large-singletask (5) & 1$\pm$0 & 2$\pm$1 & 2$\pm$1 & 1$\pm$0 & 0$\pm$0 & 0$\pm$0 & 0$\pm$0 & 2$\pm$0 & \colorbox{cyan!15}{\textbf{11$\pm$2}} & 4$\pm$2 & \colorbox{yellow!15}{5$\pm$5} & \colorbox{green!15}{8$\pm$3} \\

OGBench antsoccer-arena-singletask (5) & 1$\pm$0 & 8$\pm$2 & 0$\pm$0 & 12$\pm$4 & 1$\pm$0 & 2$\pm$4 & 12$\pm$0 & 16$\pm$1 & 33$\pm$6 & \colorbox{yellow!15}{60$\pm$2} & \colorbox{green!15}{62$\pm$20} & \colorbox{cyan!15}{\textbf{64$\pm$2}} \\

OGBench cube-single-singletask (5) & 5$\pm$1 & 83$\pm$3 & 91$\pm$2 & \colorbox{green!15}{95$\pm$2} & 80$\pm$5 & 85$\pm$9 & 81$\pm$4 & 79$\pm$7 & 79$\pm$2 & \colorbox{cyan!15}{\textbf{96$\pm$1}} & 90$\pm$6 & \colorbox{yellow!15}{94$\pm$1} \\

OGBench cube-double-singletask (5) & 2$\pm$1 & 7$\pm$1 & 12$\pm$1 & 15$\pm$6 & 2$\pm$1 & 6$\pm$2 & 5$\pm$2 & 15$\pm$3 & 14$\pm$3 & \colorbox{yellow!15}{29$\pm$1} & \colorbox{green!15}{38$\pm$22} & \colorbox{cyan!15}{\textbf{43$\pm$4}} \\

OGBench scene-singletask (5) & 5$\pm$1 & 28$\pm$1 & 41$\pm$3 & 46$\pm$3 & 20$\pm$1 & 40$\pm$7 & 30$\pm$3 & 45$\pm$5 & 30$\pm$3 & \colorbox{green!15}{\textbf{56$\pm$2}} & \colorbox{yellow!15}{51$\pm$40} & \colorbox{cyan!15}{\textbf{56$\pm$1}} \\

OGBench puzzle-3x3-singletask (5) & 2$\pm$0 & 9$\pm$1 & 21$\pm$1 & 10$\pm$2 & 18$\pm$1 & 19$\pm$0 & 6$\pm$2 & 14$\pm$4 & 19$\pm$1 & \colorbox{green!15}{30$\pm$1} & \colorbox{yellow!15}{24$\pm$38} & \colorbox{cyan!15}{\textbf{43$\pm$1}} \\

OGBench puzzle-4x4-singletask (5) & 0$\pm$0 & 7$\pm$1 & 14$\pm$1 & \colorbox{cyan!15}{\textbf{29$\pm$3}} & 10$\pm$3 & 15$\pm$3 & 1$\pm$0 & 13$\pm$1 & \colorbox{green!15}{25$\pm$5} & 17$\pm$2 & 4$\pm$2 & \colorbox{yellow!15}{15$\pm$1} \\

D4RL antmaze (6) & 17 & 57 & 78 & 79 & 74 & 30$\pm$3 & 44$\pm$3 & 64$\pm$7 & 65$\pm$7 & \colorbox{green!15}{84$\pm$3} & \colorbox{yellow!15}{81$\pm$7} & \colorbox{cyan!15}{\textbf{87$\pm$1}} \\

D4RL adroit (12) & 48 & 53 & \colorbox{cyan!15}{\textbf{59}} & \colorbox{yellow!15}{52$\pm$1} & 51$\pm$1 & 43$\pm$2 & 48$\pm$1 & 50$\pm$2 & \colorbox{green!15}{52$\pm$1} & 52$\pm$1 & 34$\pm$41 & 46$\pm$2 \\

Visual manipulation (5) & -- & 42$\pm$4 & \colorbox{yellow!15}{60$\pm$2} & -- & -- & -- & -- & 22$\pm$2 & 50$\pm$5 & \colorbox{cyan!15}{\textbf{65$\pm$2}} & 59$\pm$38 & \colorbox{green!15}{\textbf{64$\pm$1}} \\

\bottomrule
\end{tabular}
\end{adjustbox}
\end{table*}

\begin{wrapfigure}{r}{0.6\linewidth}
\centering
\includegraphics[width=\linewidth]{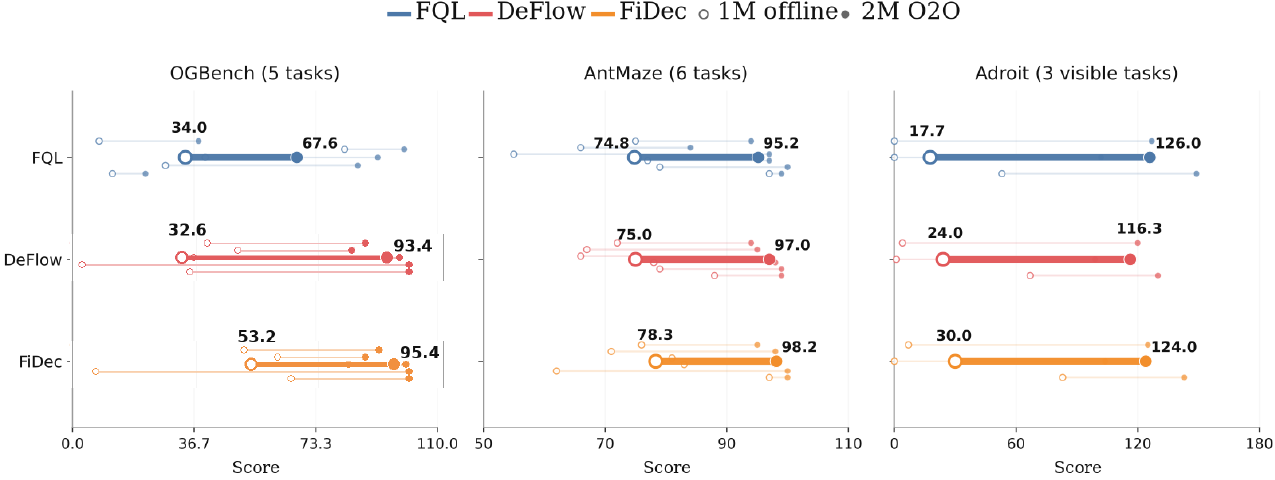}
\vspace{-10pt}
\caption{Offline-to-online fine-tuning performance. 
We evaluate 14 tasks. All methods continue training from the offline objective, with fine-tuning starting after 1M gradient steps (offline prior: $\circ$, fine-tuning: $\bullet$). Light curves denote individual tasks, and dark curves show the average.}
\label{fig:o2o_fine_tuning}
\vspace{-10pt}
\end{wrapfigure}

\textit{2. Can our new formulation in~\eqref{eq:primal-dual} go beyond isotropic regularization to enable more effective policy refinement?} \textbf{YES}. As shown in Figure~\ref{fig:examples}(a), both FQL and DeFlow exhibit limitations under isotropic regularization. FQL suffers from mode averaging, while DeFlow produces a distribution that interpolates between modes, with its mean located in low-value regions. This behavior arises because the $W_2$ distance induces an isotropic $L_2$ penalty, which treats all directions uniformly and ignores the underlying score field of the distribution, resulting in ambiguous update directions. In contrast, our formulation is anisotropic and metric-aware: the update directions are guided by the local score field through the local Fisher information matrix. This preserves multimodality while biasing updates toward favorable regions: weaker-score areas (e.g., the left mode) allow larger displacements, whereas stronger-score regions are more constrained. A similar phenomenon is observed in Figure~\ref{fig:examples}(b). The behavioral policy is supported on a thin manifold. Under isotropic regularization, DeFlow produces a more concentrated distribution but fails to reach higher-value regions, while FQL deviates from the data support and exhibits out-of-distribution drift. In contrast, FiDec follows the geometric structure and moves toward higher-value regions, demonstrating both optimality and effectiveness in policy refinement. This generally better performance in the offline setting (see Table~\ref{tab:offline_rl_latest_final}) also verifies this point; see Figure~\ref{fig:new_examples} for more example details. We provide more examples in Appendix~\ref{append:additional_result}.

\textit{3. Can our new formulation in~\eqref{eq:primal-dual} be seamlessly adopted in the offline-to-online setting?} The answer is \textbf{YES}. FiDec can be directly fine-tuned without any modification and consistently outperforms prior flow policies (Table~\ref{tab:o2o_rl_latest_final} in Appendix~\ref{append:additional_result}, Figure~\ref{fig:o2o_fine_tuning}).  We attribute this advantage to two key factors. First, FiDec provides a stronger offline prior. By parameterizing the policy as a local transport map, it retains the expressive structure of the underlying flow policy while enabling flexible local refinement. As a result, FiDec achieves a higher-quality initialization compared to prior methods, where FQL often underperforms due to its limited expressiveness under one-step distillation. Second, FiDec benefits from anisotropic metric-aware updates during fine-tuning. Guided by the local Fisher information matrix, the policy can efficiently move probability mass with lower cost in low-score regions (due to weaker score magnitude) toward more favorable regions. This leads to faster improvement and a higher concentration of samples in high-$Q$ regions, resulting in more effective online refinement.

\begin{wraptable}{l}{0.48\linewidth}
\vspace{-10pt}
\centering
\caption{Isotropic vs. anisotropic metrics under identical hyperparameters and policy parameterization under offline setting.}
\label{tab:iso_vs_aniso}
\resizebox{\linewidth}{!}{
\begin{tabular}{lcc}
\toprule
Env & Anisotropic (Fisher) & Isotropic ($L_2$) \\
\midrule
antsoccer-arena-navigate-singletask-task4      & \colorbox{cyan!15}{\textbf{0.49$\pm$0.04}} & 0.43$\pm$0.08 \\
cube-double-play-singletask-task2              & \colorbox{cyan!15}{\textbf{0.58$\pm$0.02}} & 0.50$\pm$0.03 \\
humanoidmaze-medium-navigate-singletask-task1  & \colorbox{cyan!15}{\textbf{0.76$\pm$0.07}} & 0.35$\pm$0.07 \\
puzzle-4x4-play-singletask-task4               & \colorbox{cyan!15}{\textbf{0.05$\pm$0.01}} & 0.04$\pm$0.01 \\
scene-play-singletask-task2                    & \colorbox{cyan!15}{\textbf{0.70$\pm$0.07}} & 0.51$\pm$0.06 \\
\midrule
Avg. Performance & \textbf{0.51$\pm$0.25} & 0.36$\pm$0.17 \\
\bottomrule
\end{tabular}}
\end{wraptable}

\textit{4. Is the anisotropic metric always better than the isotropic one under the same policy parameterization?} \textbf{YES}. Under the same residual displacement parameterization and training setup, we observe that the anisotropic metric induced by the Fisher information matrix consistently outperforms its isotropic counterpart ($L_2$), as shown in Table~\ref{tab:iso_vs_aniso}. This suggests that properly accounting for the local geometric properties of the behavioral policy distribution leads to more effective policy updates and improved performance, which is consistent with our theoretical analysis of the optimization gap in Appendix~\ref{append:optimal_policy_gap} from a trust-region perspective.

\begin{wraptable}{r}{0.6\linewidth}
\vspace{-12pt}
\centering
\footnotesize
\caption{Training efficiency comparison of flow policies.}
\begin{tabular}{lccc}
\toprule
Method & FQL & DeFlow & FiDec \\
\midrule
Time (ms/step) & 2.13 $\pm$ 0.14 & 2.61 $\pm$ 0.18 & \underline{2.72 $\pm$ 0.20} \\
\bottomrule
\end{tabular}
\label{tab:efficiency}
\vspace{-10pt}
\end{wraptable}

\textit{5. How efficient is FiDec compared to existing approaches?} FiDec is computationally efficient at both the algorithmic and implementation levels. It avoids expensive procedures such as ODE solving and backpropagation through time, and performs policy refinement via a simple one-step residual update. The additional cost from the Fisher information matrix is minimal, as it is computed from the flow velocity using an outer product and applied through a lightweight quadratic form in JAX. From a framework perspective, FiDec operates in a one-step optimization scheme without distillation or iterative generation, and the primal-dual updates introduce negligible overhead. As shown in Table~\ref{tab:efficiency}, FiDec achieves competitive training speed, indicating that the added anisotropic geometric structure does not compromise efficiency.

\begin{wrapfigure}{r}{0.35\linewidth}
\centering
\vspace{0pt}
\includegraphics[width=0.85\linewidth]{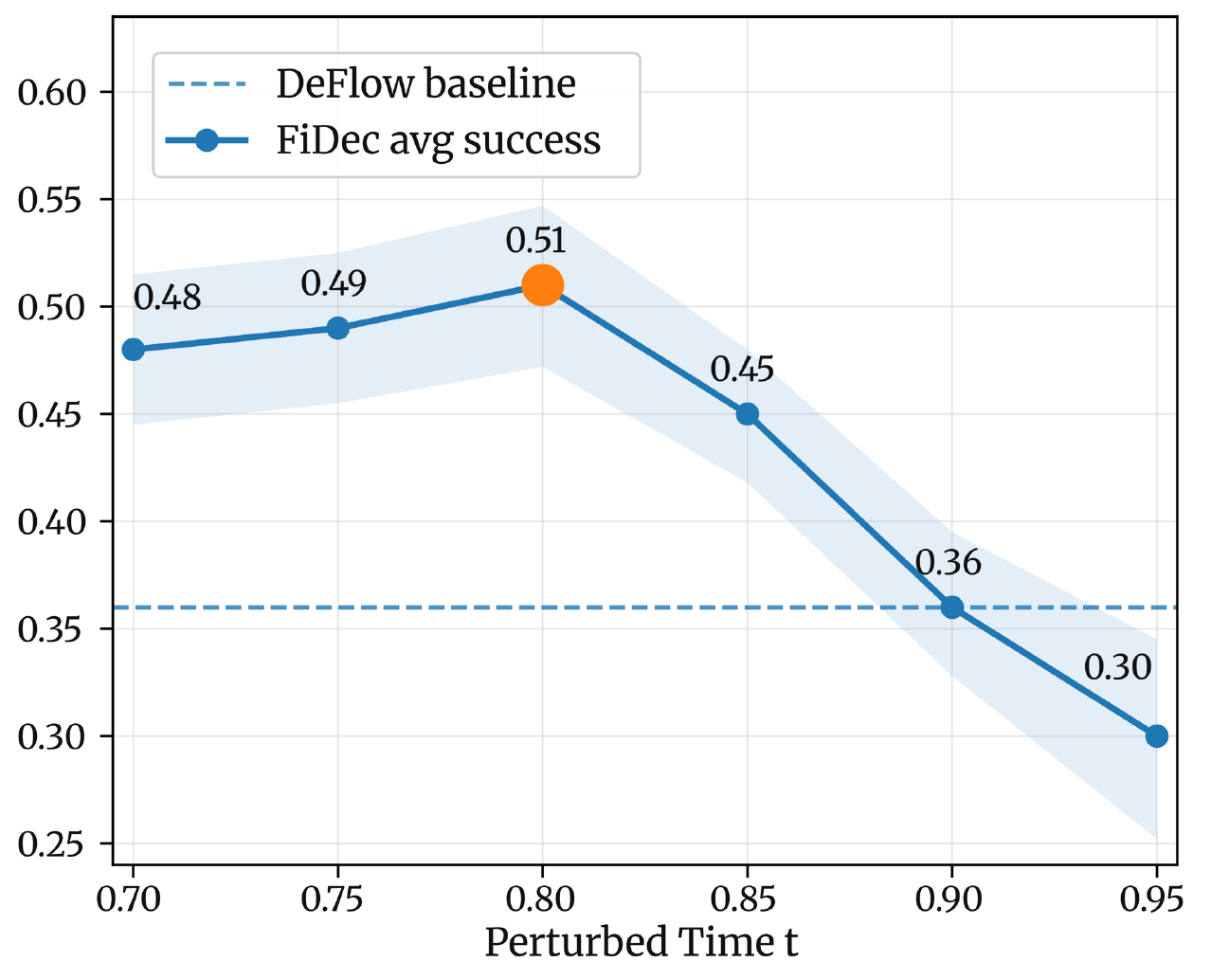}
\vspace{-5pt}
\caption{Ablations on perturbed time $t_{\varepsilon}$.}
\label{fig:perturbed_time}
\vspace{-10pt}
\end{wrapfigure}

\textit{6. Can the hyperparameter ``perturbed time $t_\varepsilon$'' be determined from first principles, rather than through heuristic tuning?} \textbf{YES}. We provide a principled characterization of the optimal perturbation by analyzing the trade-off between approximation bias and numerical error (see Appendix~\ref{append:perturbed_time}). This yields an optimal scaling $\varepsilon^* \sim \mathcal{O}(\delta_{\text{FP32}}^{\nicefrac{1}{4}})$, which depends on both machine precision $\delta$ and the spatial deformation $\left\| \nabla_a \big( a^\top \nabla_a \log \pi_\beta(a|s) \big) \right\|^2$ of the behavioral distribution.  In practice, this analysis suggests that $\varepsilon \sim \mathcal{O}(10^{-1.5})$ under standard FP32 precision, which aligns with empirical performance (Figure~\ref{fig:perturbed_time}), where the best results are achieved within the range of $t_\varepsilon\in[0.70, 0.80]$. Detailed per-task results are reported in Table~\ref{tab:perturbed_time_full},  Appendix~\ref{Append:more_results}.

\section{Conclusion}
In this paper, we identify a fundamental geometric mismatch in flow-based offline RL: KL-constrained optimization induces an anisotropic geometry, while commonly used $W_2$ surrogates reduce to isotropic regularization, leading to suboptimal updates. To address this, we formulate policy refinement as a local transport map and derive a tractable quadratic approximation of the KL objective. This yields an efficient anisotropic optimization framework, instantiated as FiDec, that enables stable and expressive policy improvement without expensive iterative procedures.  Empirically, FiDec achieves state-of-the-art performance across diverse benchmarks. Our findings underscore the importance of geometry-aware policy optimization in offline RL.


Despite its effectiveness, our method has several limitations. First, FiDec relies on an approximate estimation of the Fisher information matrix derived from the flow velocity. While efficient, this approximation may be inaccurate when the learned flow model is imperfect or when the underlying distribution exhibits highly complex structure. Second, although we provide a principled analysis of the perturbation scale, its optimal value still depends on unknown properties of the data distribution and may require mild tuning in practice. An interesting direction for future work is to further exploit the analytical formula of the optimal policy refinement. As shown in Equation~\eqref{eq:closed_form_delta} in Appendix~\ref{append:trust_region}, the optimal displacement admits a closed-form solution resembling a pointwise natural gradient update in the action space. This suggests that, given an accurate estimate of the local Fisher information matrix, policy refinement could potentially be performed without iterative optimization.

\bibliography{ref}
\bibliographystyle{unsrt}

\clearpage
\appendix
\section*{Notation}
\begin{center}
\begin{tabular}{ll} 
\toprule
\textbf{Notation} & \textbf{Meaning} \\
\midrule
$a$ & action \\
$s$ & state \\
$\mathbf{s}$ & score function \\
$t$ & time step \\
$r$ & reward function \\
$v$ & velocity field (or flow vector) \\

$\mathcal{A}$ & action space \\
$\mathcal{S}$ & state space \\
$\mathcal{B}$ & Borel $\sigma$-algebra \\
$\mathcal{D}$ & offline dataset \\

$I$ & local Fisher information matrix \\
$D_{\mathrm{KL}}$ & Kullback--Leibler divergence \\
$L_2$ & Lebesgue space equipped with $2-$norm \\
$W_2$ & 2-Wasserstein distance \\

$\mathcal{L}$ & Lagrangian \\
$\lambda$ & Lagrange multiplier (dual variable) \\

$\mathcal{N}$ & Gaussian distribution \\
$Q$ & action-value function \\
$T$ & local transport map \\

$\Xi$ & set of couplings (transport plans) \\
$\xi$ & a coupling (joint distribution) \\

$\psi$ & flow induced by a velocity field \\
$\delta$ & local displacement \\

$\mu_{\beta}$ & policy generated from the original flow \\
$\mu_{\theta}$ & policy generated from the refined flow \\

$\pi_{\beta}$ & behavior policy \\
$\pi_{\theta}$ & learned policy \\

$\#$ & pushforward operator \\
$\nabla \cdot$ & divergence operator \\

$\det(\cdot)$ & determinant of a matrix \\
$\operatorname{tr}(\cdot)$ & trace of a matrix \\
$\Box^{\dagger}$ & pseudo-inverse of a matrix \\
\bottomrule
\end{tabular}
\end{center}

\clearpage
\section*{Appendix Overview}

This appendix provides a comprehensive and self-contained presentation of the theoretical foundations, proofs, implementation details, and additional experimental results of our method.

Appendix~\ref{append:literature_review} provides provide a more comprehensive literature review covering offline RL, generative policies and residual policies. 

In Appendix~\ref{append:technical_definition}, we introduce the key technical definitions, including probability measures, Wasserstein distance, transport maps, and the local Fisher information matrix, which form the basis of our geometric formulation.  
Appendix~\ref{append:technical_proof} presents detailed derivations and proofs, including the change-of-variable formulation under transport maps, the second-order expansion of the KL divergence, and the connection to the local Fisher information matrix.  
We further analyze the trust-region optimization perspective and characterize the optimality gap induced by isotropic approximations.

Appendix~\ref{append:implementation} provides the full implementation details and algorithmic description of our method.  
Finally, Appendix~\ref{Append:more_results} includes additional experimental results, benchmark descriptions, and hyperparameter settings.

Overall, the appendix is intended to complement the main paper by providing rigorous theoretical justification and complete experimental transparency.

\clearpage
\section{More Comprehensive Literature Review} \label{append:literature_review}
We provide a more comprehensive literature review covering offline RL, generative policies and residual policies. 

\textbf{Offline RL.} Offline RL aims to learn a policy solely from previously collected data. Over the years, many offline RL methods and techniques have been proposed, most of which can be viewed through a common principle: maximizing expected return while controlling the discrepancy between the state-action distribution induced by the learned policy and that of the dataset \citep{levine2020offline, sikchi2023dual}. Prior work has instantiated this principle in various forms, including behavioral regularization \citep{ashvin2020accelerating, fujimoto2021minimalist, tarasov2023corl}, conservative value estimation \citep{kumar2020conservative}, in-sample maximization \citep{kostrikov2021offline, xu2023offline, garg2023extreme}, out-of-distribution detection \citep{yu2020mopo, nikulin2023anti}, dual RL \citep{janner2021offline}, and generative modeling \citep{chen2021decision, janner2021offline}. After offline training, the learned policy can be further improved through additional online interaction, a setting commonly known as offline-to-online RL, for which several methods have been developed \citep{lee2022offline, song2022hybrid, nakamoto2023cal, ball2023efficient, park2025flowqlearning, zhang2026reform}. Our method, FiDec, is applicable to both offline and offline-to-online RL, and introduces anisotropic constraints to guide policy improvement.

\textbf{RL with Generative Models.} Inspired by the recent success of iterative generative modeling methods, including denoising diffusion \citep{sohl2015deep, ho2020denoising, song2020score} and flow matching \citep{lipman2022flow, esser2024scaling}, a growing body of work has explored their integration into reinforcement learning. Prior studies have used iterative generative models for planning, hierarchical learning, and policy parameterization \citep{venkatraman2023reasoning, cheng2025safe, ma2025efficient, dong2025maximum}, as well as for world modeling and data augmentation \citep{farebrother2025temporal, alonso2024diffusion, hafner2025training}. Our approach falls under flow-based policy parameterization: we represent the policy with an expressive flow network capable of capturing arbitrarily complex behavior-policy distributions.

\textbf{Residual Policy.} Residual learning has a long and influential history in machine learning, ranging from ResNets \citep{he2016deep} in computer vision to residual policy learning in robotics \citep{silver2018residual}. More recent studies \citep{ankile2025imitation, xiao2025self, yuan2024policy} extend this idea by training a refinement module that adjusts actions sampled from a frozen policy through online interaction. Mu \citep{mu2026deflow} further investigate the use of residual policies within flow policies for offline reinforcement learning. In contrast, our work examines this paradigm from a different perspective: we study how the intrinsic properties of residual policies and flow matching can be leveraged to improve policy learning through an anisotropic geometric view. 

\clearpage
\section{Technical Definitions} \label{append:technical_definition}

\begin{definition}[Probability Measure  \citep{renyi2007probability}]
Let $(\mathcal{X}, \mathcal{B})$ be a measurable space. A measure $\mu : \mathcal{B} \to [0,1]$ is called a \emph{probability measure} if:
\begin{itemize}
    \item (null empty set) $\mu(\emptyset) = 0$,
    \item (countable additivity) $\mu$ is countably additive, i.e., for any sequence of disjoint sets $\{B_i\}_{i=1}^\infty \subseteq \mathcal{B}$,
    \[
    \mu\!\left(\bigcup_{i=1}^\infty B_i\right) = \sum_{i=1}^\infty \mu(B_i),
    \]
    \item (normalization) $\mu(\mathcal{X}) = 1$.
\end{itemize}
\end{definition}

\begin{definition}[2-Wasserstein Distance \citep{guionnet2004lectures}]
Let $\mu, \nu$ be probability measures on $\mathcal{X}$ with finite second moments. The 2-Wasserstein distance between $\mu$ and $\nu$ is defined as
\[
W_2(\mu, \nu) = \left( \inf_{\gamma \in \Xi(\mu, \nu)} \int_{\mathcal{X} \times \mathcal{X}} \|x - y\|^2 \, d\gamma(x,y) \right)^{1/2},
\]
where $\Xi(\mu, \nu)$ denotes the set of all couplings of $\mu$ and $\nu$.
\end{definition}

\begin{definition}[Transport Map \citep{federer2014geometric}] \label{def:transport_map}
Let $(\mathcal{X}, \mathcal{B})$ be a measurable space, and let $\mu$ and $\nu$ be probability measures on $\mathcal{X}$. A measurable map $T : \mathcal{X} \to \mathcal{X}$ is called a \emph{transport map} from $\mu$ to $\nu$ if it pushes $\mu$ forward to $\nu$, i.e.,
\[
\nu = T_{\#}\mu,
\]
meaning that for every measurable set $B \subseteq \mathcal{X}$,
\[
\nu(B) = \mu\big(T^{-1}(B)\big).
\]
Equivalently, for any measurable function $f$,
\[
\int f(y)\, d\nu(y) = \int f(T(x))\, d\mu(x).
\]
\end{definition}

\begin{remark}
In our setting, for each state $s$, the mapping $T_s: \mathcal{A} \to \mathcal{A}$ defines a transport map between action distributions. Specifically, let $\pi_\beta(\cdot | s)$ denote the well-defined behavioral policy distribution and let $a' = T_s(a)$ with $a \sim \pi_\beta(\cdot | s)$. Then the refined policy $\pi_\theta(\cdot | s)$ is given by the pushforward
\[
\pi_\theta(\cdot | s) = (T_s)_{\#} \, \pi_\beta(\cdot | s).
\]
Thus, for each fixed state $s$, the map $T_s$ transports the behavioral policy distribution $\pi_\beta(\cdot | s)$ to the refined distribution $\pi_\theta(\cdot | s)$.
\end{remark}

\begin{definition}[Local Fisher Information Matrix \citep{amari2016information}]
\label{def:local_info_matrix}
Let $p(x)$ be a smooth density on $\mathbb{R}^d$ with score function $\mathbf{s}(x) \coloneqq \nabla_x \log p(x)$.
The local Fisher information matrix is defined as
$$I(x) \coloneqq \mathbf{s}(x) \mathbf{s}(x) ^\top.$$
\end{definition}

\begin{remark}[Relation to information geometry]
The local Fisher information matrix $I(x)$ can be viewed as the local contribution (integrand) of the classical Fisher information matrix. 
$$I(p) \coloneqq \mathbb{E}_{x\sim p(x)} \mathbb{E}[\mathbf{s}(x) \mathbf{s}(x)^\top].$$
In information geometry, $I(p)$ defines a local Riemannian metric on a statistical manifold of probability distributions. However, the local object $I(x)$ itself does not directly define a Riemannian metric on the sample space, since it is local Fisher information matrix. Nevertheless, it encodes the local direction of maximal sensitivity of the log-density via the score function, and can be interpreted as the infinitesimal contribution to the local Fisher information metric under expectation.
\end{remark}

Motivated by the above, our work focuses on the cost induced by local displacements in the action space. In particular, given a displacement field $\delta(s,a)$, the associated infinitesimal cost is governed by the local Fisher information matrix through
$$(\delta(s,a)^\top \mathbf{s}(a|s))^2 = \delta(s,a)^\top I(s,a)\delta(s,a),$$
where $I(s,a) = \mathbf{s}(a|s)\mathbf{s}(a|s)^{\top}$ is the local Fisher information matrix.
This reveals that the underlying anisotropic geometric structure induced by the distribution is inherently anisotropic: \textit{only directions aligned with the score function $\mathbf{s}(a|s)$ contribute to the cost, while directions orthogonal to $\mathbf{s}(a|s)$ incur no first-order variation. Consequently, the geometry is highly directional and depends on the local behavior of the log-density.}

\clearpage
\section{Technical Proofs} \label{append:technical_proof}

\subsection{Refined Policy Distribution under Transport Map $T$} 

\begin{lemma}[Change of Variables under Transport Maps] \label{lemma:change_of_variable_transport}
For all state $s\in \mathcal{S}$, let $T_s: \mathcal{A} \to \mathcal{A}$ be an invertible and differentiable transport map for a fixed state $s$, and let $\pi_\beta(\cdot | s)$ be a probability distribution on $\mathcal{A}$. Define the pushforward distribution
\[
\pi_\theta(\cdot | s) = (T_s)_{\#} \, \pi_\beta(\cdot | s).
\]
Then the density of $\pi_\theta$ is given by the change-of-variables formula:
\[
\pi_{\theta}(a' | s)
=
\pi_\beta\big(T^{-1}_s(a')| s\big)
\cdot
\left|\det \nabla_{a'} T^{-1}_s(a')\right|,
\]
where $a' = T_s(a)$ and $\nabla_{a'} T^{-1}_s(a')$ denotes the Jacobian of the inverse map with respect to its input.
\end{lemma}

\begin{proof}
Fix an arbitrary state $s \in \mathcal{S}$ and let $T_s$ be an invertible and differentiable map. 
By definition of transport map in Definition~\ref{def:transport_map}, for any measurable set $B \subseteq \mathcal{A}$,
\[
\pi_\theta(B | s)
=
\pi_\beta\big(T^{-1}_s(B)| s\big).
\]
Assume that $\pi_\beta(\cdot | s)$ admits a density (still denoted by $\pi_\beta$). Then
\[
\pi_\theta(B | s)
=
\int_{T^{-1}_s(B)} \pi_\beta(a | s)\, da.
\]
Applying the change-of-variables formula with $a' = T(s,a)$ yields
\[
\int_{T^{-1}_s(B)} \pi_\beta(a | s)\, da
=
\int_{B} \pi_\beta\big(T^{-1}_s(a')| s\big)
\left|\det \nabla_{a'} T^{-1}_s(a')\right| \, da'.
\]
Since this holds for all measurable sets $B$, we conclude that the density of $\pi_\theta(\cdot | s)$ is
\[
\pi_{\theta}(a' | s)
=
\pi_\beta\big(T^{-1}_s(a')| s\big)
\cdot
\left|\det \nabla_{a'} T^{-1}_s(a')\right|.
\]
\end{proof}
\begin{lemma} \label{lemma:det_approx}
Let 
\begin{align} 
    T_s(a) = a + \delta(s,a), \quad \delta(s, \cdot) \in C^1(\mathbb{R}^d, \mathbb{R}^d)
\end{align}
Assume that $T$ is invertible and $\|\nabla_a \delta(s,a)\|$ is small. Then, for all $a' = T_s(a)$, we have
\begin{equation}
    \left|\det \nabla_{a'} T^{-1}_s(a')\right|
    =
    1 - \nabla _a\cdot \delta(s,a)
    +
    \mathcal{O}(\|\nabla_a \delta(s,a)\|^2),
\end{equation}
where $a = T^{-1}_s(a')$.
\end{lemma}

\begin{proof}
The Jacobian of $T$ is given by
\begin{equation}
    \nabla_a T_s(a) = I + \nabla_a \delta(s,a).
\end{equation}
Using the determinant expansion
\begin{equation}
    \det(I + A)
    =
    1 + \operatorname{tr}(A)
    +
    \underbrace{\frac{1}{2}\left[(\operatorname{tr}A)^2 - \operatorname{tr}(A^2)\right]
    + \mathcal{O}(\|A\|^3)}_{\text{high-order terms}},
\end{equation}
we obtain
\begin{equation}
    \det \nabla_a T_s(a)
    =
    1 + \nabla \cdot \delta(s,a)
    + \mathcal{O}(\|\nabla_a \delta(s,a)\|^2).
\end{equation}
By the inverse function theorem,
\begin{equation}
    \det \nabla_{a'} T^{-1}_s(a')
    =
    \frac{1}{\det \nabla_a T_s(a)}.
\end{equation}
Applying the Taylor expansion $(1+\epsilon)^{-1} = 1 - \epsilon + \mathcal{O}(\epsilon^2)$ yields
\begin{equation} \label{eq:final_lemma2}
    |\det \nabla_{a'} T^{-1}_s(a') | 
    =
    1 - \nabla \cdot \delta(s,a)
    + \mathcal{O}(\|\nabla_a \delta(s,a)\|^2),
\end{equation}
where $a = T^{-1}_s(a')$. 
\end{proof}
\begin{theorem}[Second-order expansion of KL divergence under small transport] \label{theorem:second_order_expansion}
Let $\pi_\beta(\cdot | s)$ be a probability density on $\mathbb{R}^d$ with $\log \pi_\beta \in C^2(\mathbb{R}^d, \mathbb{R})$, and define a transport map
\begin{align}
    T_s(a) = a + \delta(s,a),
\end{align}
where $\delta(s,\cdot) \in C^1(\mathbb{R}^d, \mathbb{R}^d)$ and $\|\nabla_a \delta(s,a)\|$ is small so that $T$ is invertible.

Let $\pi_\theta(\cdot | s)$ be the pushforward of $\pi_\beta(\cdot | s)$ under $T$, i.e.,
\begin{align}
    \pi_\theta(a' | s)
    =
    \pi_\beta(a | s)
    \left|\det \nabla_{a'} T^{-1}_s(a')\right|,
    \quad a' = T(s,a).
\end{align}
Assume that $\pi_\beta$ decays sufficiently fast at infinity so that integration by parts is valid.

Then, the Kullback--Leibler divergence admits the expansion
\begin{align}
    D_{\mathrm{KL}}\big(\pi_\theta(\cdot | s)\,\|\,\pi_\beta(\cdot | s)\big)
    \approx
    \frac{1}{2}
    \int_{\mathbb{R}^d}
    \pi_\beta(a | s)\,
    \delta(s,a)^\top
    \nabla_a \log \pi_\beta(a | s )\nabla_a \log \pi_\beta(a | s)^{\top}\,
    \delta(s,a)
    \, da.
\end{align}
\end{theorem}

\begin{proof}
We start from the definition of the Kullback--Leibler divergence:
\begin{align}
    D_{\mathrm{KL}}(\pi_\theta \| \pi_\beta)
    =
    \int_{\mathbb{R}^d}
    \pi_\theta(a' | s)
    \log \frac{\pi_\theta(a' | s)}{\pi_\beta(a' | s)}
    \, da'.
\end{align}

Let $a' = T(s,a) = a + \delta(s,a)$. Using the pushforward relation in Lemma~\ref{lemma:change_of_variable_transport}
\begin{align}
    \pi_\theta(a' | s)
    =
    \pi_\beta(a | s)
    \left|\det \nabla_a T_s(a)\right|^{-1},
\end{align}
we obtain
\begin{align}
    D_{\mathrm{KL}}
    =
    \int_{\mathbb{R}^d}
    \pi_\beta(a | s)
    \log \frac{\pi_\theta(T_s(a) | s)}{\pi_\beta(T_s(a) | s)}
    \, da.
\end{align}

Substituting the density transformation yields
\begin{align}
    \log \frac{\pi_\theta(T_s(a)| s)}{\pi_\beta(T_s(a)| s)}
    =
    \log \pi_\beta(a | s)
    - \log \pi_\beta(T_s(a)| s)
    - \log \det \nabla_a T_s(a).
\end{align}

We now expand each term up to second order.

\vspace{0.5em}
\noindent
\textbf{Step 1: Expansion of $\log \pi_\beta(T_s(a))$.}

Using Taylor expansion,
\begin{align}
    \log \pi_\beta(a + \delta)
    =
    \log \pi_\beta(a|s)
    +
    \nabla_a \log \pi_\beta(a|s)^\top \delta
    +
    \frac{1}{2}
    \delta^\top \nabla^2_a \log \pi_\beta(a|s)\, \delta
    +
    \mathcal{O}(\|\delta\|^3).
\end{align}

\vspace{0.5em}
\noindent
\textbf{Step 2: Second-order expansion of the log-determinant.}

Using the exact Taylor expansion for the log-determinant of the Jacobian $\nabla_a T_s(a) = I + \nabla_a \delta(s,a)$, we have
\begin{align}
    \log \det \nabla_a T_s(a)
    &=
    \operatorname{tr}(\nabla_a \delta(s,a)) 
    - \frac{1}{2} \operatorname{tr}\big((\nabla_a \delta(s,a))^2\big) 
    + \mathcal{O}(\|\nabla_a \delta\|^3) \nonumber \\
    &=
    \nabla \cdot \delta(s,a) 
    - \frac{1}{2} \operatorname{tr}\big((\nabla_a \delta(s,a))^2\big) 
    + \mathcal{O}(\|\nabla_a \delta\|^3).
\end{align}
Unlike the first-order approximation, we must retain the quadratic trace term as it contributes to the second-order expansion of the KL divergence.

\vspace{0.5em}
\noindent
\textbf{Step 3: Combine expansions.}

Substituting the expansions of both the log-density and the log-determinant into the log-ratio, we obtain
\begin{align}
    \log \frac{\pi_\theta(T_s(a)|s)}{\pi_\beta(T_s(a)|s)}
    =
    &- \nabla_a \log \pi_\beta(a|s)^\top \delta
    - \frac{1}{2} \delta^\top \nabla^2_a \log \pi_\beta(a|s)\, \delta \nonumber \\
    &- \nabla \cdot \delta
    + \frac{1}{2} \operatorname{tr}\big((\nabla_a \delta)^2\big)
    + \mathcal{O}(\|\delta\|^3).
\end{align}

\vspace{0.5em}
\noindent
\textbf{Step 4: Substitute into KL.}

Integrating over the base distribution $\pi_\beta(a|s)$, the KL divergence becomes:
\begin{align}
    D_{\mathrm{KL}}\big(\pi_\theta(\cdot | s)\,\|\,\pi_\beta(\cdot | s)\big)
    =
    \int \pi_\beta(a|s)
    \bigg[
    &- \nabla_a \log \pi_\beta(a|s)^\top \delta
    - \nabla \cdot \delta \nonumber \\
    &- \frac{1}{2} \delta^\top \nabla^2_a \log \pi_\beta(a|s)\, \delta
    + \frac{1}{2} \operatorname{tr}\big((\nabla_a \delta)^2\big)
    \bigg]
    da
    + \mathcal{O}(\|\delta\|^3).
\end{align}

\vspace{0.5em}
\noindent
\textbf{Step 5: Cancellation of first-order terms.}

By applying integration by parts and assuming $\pi_\beta$ decays sufficiently fast at the boundary \citep{neuberger2009sobolev, brezis2011functional}, the expected divergence of the vector field equates to:
\begin{align}
    \int \pi_\beta \, \nabla \cdot \delta \, da
    =
    - \int \nabla_a \pi_\beta \cdot \delta \, da
    =
    - \int \pi_\beta \nabla_a \log \pi_\beta \cdot \delta \, da.
\end{align}
Consequently, the first-order terms exactly cancel out:
\begin{align}
    \int \pi_\beta
    \Big(
    - \nabla_a \log \pi_\beta \cdot \delta
    - \nabla \cdot \delta
    \Big) da = 0.
\end{align}

\vspace{0.5em}
\noindent
\textbf{Step 6: Final result and reduction to Fisher form.}

We are left with the second-order terms comprising the Hessian of the log-density and the trace of the squared Jacobian:
\begin{align} \label{eq:kl_second_order_raw}
    D_{\mathrm{KL}}
    =
    \frac{1}{2}
    \int \pi_\beta
    \Big[
    - \delta^\top \nabla^2_a \log \pi_\beta \, \delta
    + \operatorname{tr}\big((\nabla_a \delta)^2\big)
    \Big] da
    + \mathcal{O}(\|\delta\|^3).
\end{align}

Using the logarithmic derivative identity, we decompose the negative Hessian into two components:
\begin{align}
    -\nabla^2_a \log \pi_\beta
    =
    - \underbrace{\frac{\nabla^2_a \pi_\beta}{\pi_\beta}}_{\text{local curvature}}
    +
    \underbrace{\nabla_a \log \pi_\beta \nabla_a \log \pi_\beta^\top}_{\text{Fisher information matrix}}.
\end{align}
Substituting this into \eqref{eq:kl_second_order_raw} yields:
\begin{align} \label{eq:kl_split}
    D_{\mathrm{KL}}
    =
    \frac{1}{2} \int \pi_\beta \Big( \delta^\top \nabla_a \log \pi_\beta \nabla_a \log \pi_\beta^\top \delta \Big) da
    - \frac{1}{2} \int \delta^\top (\nabla^2_a \pi_\beta) \delta \, da
    + \frac{1}{2} \int \pi_\beta \operatorname{tr}\big((\nabla_a \delta)^2\big) da.
\end{align}

To resolve the curvature term $\int \delta^\top (\nabla^2_a \pi_\beta) \delta \, da$, we apply integration by parts. Since $\nabla_a \pi_\beta = \pi_\beta \nabla_a \log \pi_\beta$, the boundary-free integral transfers the derivative onto the displacement field:
\begin{align} \label{eq:curvature_by_parts}
    \int \delta^\top (\nabla^2_a \pi_\beta) \delta \, da
    =
    - \int \pi_\beta \Big[ (\nabla \cdot \delta)(\delta^\top \nabla_a \log \pi_\beta) + \delta^\top (\nabla_a \delta) \nabla_a \log \pi_\beta \Big] da.
\end{align}

Inserting \eqref{eq:curvature_by_parts} back into \eqref{eq:kl_split}, we can group all terms involving the Jacobian $\nabla_a \delta$ into a single residual operator $\mathcal{E}$:
\begin{align}
    D_{\mathrm{KL}}
    =
    \frac{1}{2} \int \pi_\beta \Big( \delta^\top \nabla_a \log \pi_\beta \nabla_a \log \pi_\beta^\top \delta \Big) da + \mathcal{E}(\delta, \nabla_a \delta),
\end{align}
where
\begin{align}
    \mathcal{E} = \frac{1}{2} \int \pi_\beta \Big[ \operatorname{tr}\big((\nabla_a \delta)^2\big) + (\nabla \cdot \delta)(\delta^\top \nabla_a \log \pi_\beta) + \delta^\top (\nabla_a \delta) \nabla_a \log \pi_\beta \Big] da.
\end{align}

In the context of residual policy learning, we naturally regularize the transport map $\delta(s,a)$ to be a \textbf{slowly varying, locally near-rigid displacement field} to avoid catastrophic distributional shift. This physical constraint implies that the local deformation is vanishingly small compared to the translation itself, i.e., $\| \nabla_a \delta \| = o(\| \delta\|)$.

Under this setting, the residual $\mathcal{E}$, which scales as $\mathcal{O}(\|\nabla_a \delta\|^2)$ and $\mathcal{O}(\|\delta\| \|\nabla_a \delta\|)$, is dominated by the translational energy $\mathcal{O}(\|\delta\|^2)$. Omitting these higher-order deformational artifacts, we recover the elegant Fisher information form:
\begin{align}
    D_{\mathrm{KL}}\big(\pi_\theta(\cdot | s)\,\|\,\pi_\beta(\cdot | s)\big)
    \approx
    \frac{1}{2}
    \int
    \pi_\beta(a | s)
    \,
    \delta(s,a)^\top
    \big(
     \nabla_a \log \pi_\beta(a | s)
    \nabla_a \log \pi_\beta(a | s)^\top
    \big)
    \delta(s,a)
    \, da.
\end{align}

This result fundamentally demonstrates that constraining the Jacobian of the displacement field naturally reduces the KL divergence to a local quadratic form governed by the local Fisher information metric, measuring the translational discrepancy.
\end{proof}

\textbf{Miscellaneous Complements: Formal Analysis of the Curvature Term}

We provide a rigorous analysis of the term $\mathcal{J} = \int_{\mathcal{A}} \delta^\top (\nabla^2_a \pi_\beta) \delta \, da$ and its physical implications for the transport map in refining flow policies. 

\paragraph{1. Global Cancellation of Density Curvature}
We first establish that the total curvature of any valid probability density integrates to zero over its support. Let $\mathcal{A} \subseteq \mathbb{R}^d$ be the action space. By the definition of the Laplacian operator, we have $\nabla_a^2 \pi_\beta = \nabla \cdot (\nabla_a \pi_\beta)$. According to the Divergence Theorem (or Stokes' Theorem) \citep{macdonald2012vector}, the integral over $\mathcal{A}$ can be converted into a boundary integral:
\begin{align}
    \int_{\mathcal{A}} \nabla^2_a \pi_\beta \, da = \int_{\mathcal{A}} \nabla \cdot (\nabla_a \pi_\beta) \, da = \oint_{\partial \mathcal{A}} (\nabla_a \pi_\beta \cdot \vec{n}) \, dS,
\end{align}
where $\vec{n}$ denotes the unit outward normal vector on the boundary $\partial \mathcal{A}$ of the compact set $\mathcal{A}$. For behavioral policy distributions, we assume $\pi_\beta$ decays sufficiently fast at infinity (or reaches zero at the boundary of the compact support), implying:
\begin{itemize}
    \item $\pi_\beta(a|s) \to 0$ as $a \to \partial \mathcal{A}$,
    \item $\nabla_a \pi_\beta(a|s) \to 0$ (the density becomes flat at the boundary).
\end{itemize}
Under these regularity conditions, the boundary flux vanishes, leading to the global cancellation of curvature:
\begin{equation} \label{eq:cancellation}
    \int_{\mathcal{A}} \nabla_a^2 \pi_\beta \, da = 0.
\end{equation}

\paragraph{2. Physical Insight: Avoiding Distributional Shift}
In our residual policy setting, we refine the behavior policy $\pi_\beta$ via a displacement field $\delta(s,a)$. The curvature term $\frac{\nabla^2_a \pi_\beta}{\pi_\beta}$ represents local \textbf{volumetric deformation}—the degree to which the action space is locally compressed or stretched. 

Equation \eqref{eq:cancellation} shows that the expectation of this curvature over the entire action space is zero. When designing the residual term $\delta(s,a)$, our objective is to ensure that $\nabla_a \delta(s,a)$ does not change dramatically. A high gradient in $\delta$ would imply a significant deformation of the behavioral policy structure, potentially leading to a catastrophic \textbf{distributional shift}. 

By ensuring $\|\nabla_a \delta(s,a) \|$ small relatively to $\| \delta(s,a)\|$, we allow to transport its high-probability mass while preserving the internal density structure (i.e., avoiding "collapsing" or "tearing" the distribution). Under such smoothness assumptions, $\delta$ can be treated as locally constant relative to the curvature, yielding:
\begin{equation}
    \mathcal{J} = \int_{\mathcal{A}} \delta^\top (\nabla^2_a \pi_\beta) \delta \, da \approx 0.
\end{equation}
This supports the use of the Fisher information form as the leading-order metric for policy divergence: it captures the translational energy of the update, while the deformation-related curvature effects are asymptotically negligible under the near-rigid transport.

\subsection{Score Function Estimation}
\begin{lemma} \label{lemma:score_computation}
Let $a^1 \sim \pi_{\beta}(a|s)$ and $z \sim \mathcal{N}(0, I_d)$ be independent random variables. Define the convex combination:
\[
a^t = t a^1 + (1-t) z, \quad t \in (0,1).
\]
Then, the score function of the marginalized distribution $\pi_\beta(a^t|s)$ is given by:
\[
\nabla_{a} \log \pi_\beta(a^t|s) = \frac{t\, \mathbb{E}[a^1 | a^t] - a^t}{(1-t)^2}.
\]
Furthermore, given the conditional velocity field $v(t, s, a^t)$ defined in the flow matching framework, the score function can be equivalently expressed as:
\begin{equation} \label{eq:score_function}
\nabla_{a} \log \pi_\beta(a^t|s) = \frac{t v(t, s, a^t) - a^t}{1-t}.
\end{equation}
\end{lemma}

\begin{proof}
Based on the linear Gaussian interpolation $a^t = t a^1 + (1-t)z$, the conditional distribution of $a^t$ given $a^1$ is:
\[
a^t | a^1 \sim \mathcal{N}(t a^1, (1-t)^2 I_d).
\]
The log-likelihood gradient (score) of this conditional Gaussian density is:
\begin{equation}
\nabla_{a^t} \log p(a^t | a^1, s) = \nabla_{a^t} \log \mathcal{N}\left(a^t; t a^1, (1-t)^2 I_d\right) = -\frac{a^t - t a^1}{(1-t)^2}.
\end{equation}
By the property of the marginal score function (or via the law of total expectation applied to the score):
\begin{align*}
\nabla_{a} \log \pi_\beta(a^t|s) &= \mathbb{E}_{p(a^1|a^t, s)} \left[ \nabla_{a} \log p(a^t|a^1, s) \right] \\
&= \mathbb{E} \left[ \left. -\frac{a^t - t a^1}{(1-t)^2} \right| a^t \right] \\
&= \frac{t \mathbb{E}[a^1 | a^t] - a^t}{(1-t)^2}.
\end{align*}
In the context of Flow Matching with a Gaussian path, the optimal velocity field $v(t, s, a^t)$ satisfies the relationship:
\begin{equation}
v(t, s, a^t) = \mathbb{E} \left[ \left. \frac{d}{dt} a^t \right| a^t \right] = \mathbb{E}[a^1 - z | a^t].
\end{equation}
Substituting $z = \frac{a^t - t a^1}{1-t}$ into the above, we obtain the identity:
\begin{equation}
\mathbb{E}[a^1 | a^t] = a^t + (1-t) v(t, s, a^t).
\end{equation}
Substituting this into the score expression:
\begin{align*}
\nabla_{a} \log \pi_\beta(a^t|s) &= \frac{t \left( a^t + (1-t) v(t, s, a^t) \right) - a^t}{(1-t)^2} \\
&= \frac{(t-1) a^t + t(1-t) v(t, s, a^t)}{(1-t)^2} \\
&= \frac{t v(t, s, a^t) - a^t}{1-t}.
\end{align*}
This completes the proof.
\end{proof}

Our objective is to recover the score function of the behavioral distribution as $t \to 1$. However, the expression in \eqref{eq:score_function} exhibits a singularity in this limit, as both the numerator and denominator vanish simultaneously, leading to numerical instability. To circumvent this, we introduce a small perturbation $\varepsilon$ to the time variable, yielding a numerically stable approximation. The following theorem establishes the convergence rate and error bound between this perturbed approximation and the ground-truth score function.

\begin{theorem}[Pointwise Convergence of the Perturbed Score Function] \label{theorem:perturbation_error}
    Let $\pi_\beta (a|s)$ denote the behavioral policy distribution defined on a compact set of $\mathbb{R}^d,$ and assume $\pi_\beta (\cdot|s) \in C^3(\mathbb{R}^d, \mathbb{R})$ over the compact subset of $\mathbb{R}^d$. Consider the linear Gaussian interpolation path, let $\mathbf{s}_t(a|s)$ be the associated score function at time $t$. For a small perturbation parameter $\varepsilon > 0$, define the perturbed time $t_\varepsilon = 1 - \varepsilon$. As $\varepsilon \to 0^+$, the approximation error between the perturbed score $\mathbf{s}_{1-\varepsilon}(a|s)$ and the ground truth score $\mathbf{s}_1(a|s)$ satisfies: 
    $$ \| \mathbf{s}_{1-\varepsilon}(a|s) - \mathbf{s}_1(a|s) \| = \varepsilon \left\| \nabla_a \big( a^\top \nabla_a \log \pi_\beta(a|s) \big) \right\| + \mathcal{O}(\varepsilon^2). $$ 
    Consequently, the boundary perturbation scheme achieves a first-order convergence rate of $\mathcal{O}(\varepsilon)$.
\end{theorem}

\begin{proof}
\textbf{1. Integral Representation with Drift} \\
In the linear Gaussian interpolation path $a_t = t a_0 + (1-t) \eta$, where $a_0 \sim \pi_\beta(\cdot|s)$ and $\eta \sim \mathcal{N}(0, I_d)$, the marginal density $p_t(a|s)$ at time $t = 1-\varepsilon$ is given by:
\begin{equation*}
    p_{1-\varepsilon}(a|s) = \int \pi_\beta(a_0|s) \mathcal{N}\big(a; (1-\varepsilon) a_0, \varepsilon^2 I_d\big) da_0.
\end{equation*}
To evaluate this convolution, we reparameterize the integral by setting $y = (1-\varepsilon)a_0$. The change of measure yields $da_0 = (1-\varepsilon)^{-d} dy$, giving:
\begin{equation*}
    p_{1-\varepsilon}(a|s) = \int \underbrace{(1-\varepsilon)^{-d} \pi_\beta\left(\frac{y}{1-\varepsilon}\bigg|s\right)}_{:= f_\varepsilon(y)} \mathcal{N}(a; y, \varepsilon^2 I_d) dy.
\end{equation*}

\vspace{0.5em}
\textbf{2. Taylor Expansion of the Base Density} \\
We first expand the scaled density function $f_\varepsilon(y)$ for small $\varepsilon$. Using $(1-\varepsilon)^{-d} = 1 + d\varepsilon + \mathcal{O}(\varepsilon^2)$ and the Taylor expansion of $\pi_\beta$, we have:
\begin{align*}
    f_\varepsilon(a) &= (1 + d\varepsilon) \left[ \pi_\beta(a|s) + \nabla_a \pi_\beta(a|s)^\top (\varepsilon a) \right] + \mathcal{O}(\varepsilon^2) \\
    &= \pi_\beta(a|s) + \varepsilon \big( d \pi_\beta(a|s) + a^\top \nabla_a \pi_\beta(a|s) \big) + \mathcal{O}(\varepsilon^2).
\end{align*}
Recognizing the term in the parenthesis as the divergence of the vector field $a \pi_\beta(a|s)$, we obtain:
\begin{equation*}
    f_\varepsilon(a) = \pi_\beta(a|s) + \varepsilon \nabla_a \cdot \big( a \pi_\beta(a|s) \big) + \mathcal{O}(\varepsilon^2).
\end{equation*}

\vspace{0.5em}
\textbf{3. Applying the Gaussian Mollifier} \\
The integral $p_{1-\varepsilon}(a|s) = \int f_\varepsilon(y) \mathcal{N}(a; y, \varepsilon^2 I_d) dy$ is the standard heat equation evolution of $f_\varepsilon$ for ``heat-time'' $\tau = \varepsilon^2$. Thus,
\begin{equation*}
    p_{1-\varepsilon}(a|s) = f_\varepsilon(a) + \frac{\varepsilon^2}{2} \Delta_a f_\varepsilon(a) + \mathcal{O}(\varepsilon^4).
\end{equation*}
Substituting our expansion for $f_\varepsilon(a)$, the $\varepsilon^2$ term from the heat kernel is absorbed into the higher-order terms, leaving the drift-induced first-order error as the dominant term:
\begin{equation} \label{eq:p_expand}
    p_{1-\varepsilon}(a|s) = \pi_\beta(a|s) + \varepsilon \nabla_a \cdot \big( a \pi_\beta(a|s) \big) + \mathcal{O}(\varepsilon^2).
\end{equation}

\vspace{0.5em}
\textbf{4. Expansion of the Score Function} \\
Taking the gradient of \eqref{eq:p_expand} with respect to $a$, we get:
\begin{equation*}
    \nabla_a p_{1-\varepsilon} = \nabla_a \pi_\beta + \varepsilon \nabla_a \big( \nabla_a \cdot (a \pi_\beta) \big) + \mathcal{O}(\varepsilon^2).
\end{equation*}
The score function is $\mathbf{s}_{1-\varepsilon} = \frac{\nabla_a p_{1-\varepsilon}}{p_{1-\varepsilon}}$. Using the first-order quotient linearization $\frac{u + \varepsilon \delta u}{v + \varepsilon \delta v} \approx \frac{u}{v} + \varepsilon \left( \frac{v \delta u - u \delta v}{v^2} \right)$, where $u = \nabla_a \pi_\beta$ and $v = \pi_\beta$, we have:
\begin{equation*}
    \mathbf{s}_{1-\varepsilon} = \frac{\nabla_a \pi_\beta}{\pi_\beta} + \varepsilon \left[ \frac{\pi_\beta \nabla_a \big( \nabla_a \cdot (a \pi_\beta) \big) - \nabla_a \pi_\beta \big( \nabla_a \cdot (a \pi_\beta) \big)}{\pi_\beta^2} \right] + \mathcal{O}(\varepsilon^2).
\end{equation*}
The term in the bracket is exactly the gradient of the scalar field $\frac{\nabla_a \cdot (a \pi_\beta)}{\pi_\beta}$ via the quotient rule. Therefore:
\begin{equation*}
    \mathbf{s}_{1-\varepsilon}(a|s) - \mathbf{s}_1(a|s) = \varepsilon \nabla_a \left( \frac{\nabla_a \cdot (a \pi_\beta)}{\pi_\beta} \right) + \mathcal{O}(\varepsilon^2).
\end{equation*}

\vspace{0.5em}
\textbf{5. Final Simplification} \\
We can simplify the scalar field inside the gradient:
\begin{equation*}
    \frac{\nabla_a \cdot (a \pi_\beta)}{\pi_\beta} = \frac{d \pi_\beta + a^\top \nabla_a \pi_\beta}{\pi_\beta} = d + a^\top \nabla_a \log \pi_\beta(a|s).
\end{equation*}
Since the gradient of the constant $d$ is zero, the error term reduces to:
\begin{equation*}
    \mathbf{s}_{1-\varepsilon}(a|s) - \mathbf{s}_1(a|s) = \varepsilon \nabla_a \big( a^\top \nabla_a \log \pi_\beta(a|s) \big) + \mathcal{O}(\varepsilon^2).
\end{equation*}
Taking the norm completes the proof.
\end{proof}
\subsection{Analysis of Optimal Choice of Time Perturbation} 
\label{append:perturbed_time}

According to the approximation of the score function, the local quadratic form (or Fisher information equivalent) can be computed as an outer product:
\begin{equation}
    I(s,a) \approx \frac{\big( (1-\varepsilon)v_{\beta}(1-\varepsilon, s, a) - a \big) \big( (1-\varepsilon)v_{\beta}(1-\varepsilon, s, a) - a \big)^{\top}}{\varepsilon^2}. \label{eq:score_approx}
\end{equation}
We now turn to analyze the optimal choice of the time perturbation $\varepsilon$. The total computation error, denoted by $\mathcal{E}_{\text{total}}(\varepsilon) = \mathcal{E}_{\text{trunc}} + \mathcal{E}_{\text{num}}$, consists of two primary components: the approximation error stemming from the time perturbation (truncation bias) and the numerical rounding error that arises as $\varepsilon$ becomes vanishingly small.

Based on the first-order pointwise convergence established in Theorem~\ref{theorem:perturbation_error}, the norm of the score error scales as $\mathcal{O}(\varepsilon)$. Consequently, the \textbf{truncation error} for the squared objective $I(s,a)$ satisfies:
\begin{equation}
    \mathcal{E}_{\text{trunc}} \approx C_{1} \varepsilon^2, \quad \text{where} \quad C_{1} = \left\| \nabla_a \big( a^\top \nabla_a \log \pi_\beta(a|s) \big) \right\|^2.
\end{equation}
The constant $C_1$ captures the magnitude of the "tidal force" or spatial deformation of the score field under radial shrinkage. A distribution with highly varying local curvature implies that the score field deforms sharply towards the origin, necessitating a smaller $\varepsilon$ to minimize this truncation bias.

Conversely, the \textbf{numerical rounding error} is a consequence of finite floating-point precision (machine epsilon $\delta$). In the expression \eqref{eq:score_approx}, calculating the squared norm involves dividing by $\varepsilon^2$. As $\varepsilon \to 0$, the denominator amplifies the infinitesimal floating-point variance in the numerator. This numerical noise variance can be modeled as:
\begin{equation}
    \mathcal{E}_{\text{num}} \approx \frac{C_2 \delta}{\varepsilon^2},
\end{equation}
where $C_2$ is a constant related to the magnitude of the velocity field $v$ and the action $a$. To identify the optimal perturbation $\varepsilon^*$, we minimize the total mean squared error:
\begin{equation}
    \min_{\varepsilon} \mathcal{E}_{\text{total}}(\varepsilon) = C_1 \varepsilon^2 + \frac{C_2 \delta}{\varepsilon^2}.
\end{equation}
Taking the derivative with respect to $\varepsilon$ and setting it to zero yields:
\begin{equation}
    2 C_1 \varepsilon - \frac{2C_2 \delta}{\varepsilon^3} = 0 \implies \varepsilon^* = \left( \frac{C_2 \delta}{C_1} \right)^{1/4}.
\end{equation}

This derivation reveals a profound connection to classical numerical analysis: the optimal time perturbation $\varepsilon^*$ scales with the machine precision following the classical first-order finite-difference scaling law, $\mathcal{O}(\delta^{1/4})$. Furthermore, $\varepsilon^*$ is inversely proportional to the fourth root of the score field's radial deformation magnitude ($C_1$). 

Consequently, while highly complex distributions theoretically require a finer perturbation, the practical lower bound of $\varepsilon$ is strictly dictated by the hardware precision $\delta_{\text{FP32}}$. In a standard FP32 environment ($\delta_{\text{FP32}} \sim 10^{-6}$), the optimal perturbation scales near $(10^{-6})^{1/4} \approx 10^{-1.5}$. Therefore, choosing $\varepsilon \sim \mathcal{O}(10^{-2})$ to $\mathcal{O}(10^{-1})$ typically achieves the most robust balance between approximation accuracy and numerical stability. Empirical results from the ablation studies are completely consistent with this theoretical scaling law, as shown in Figure~\ref{fig:perturbed_time} and Table~\ref{tab:perturbed_time_full}.
\subsection{Trust-Region Problem} \label{append:trust_region}

The detailed algorithm is presented in Appendix~\ref{append:implementation}, and the following paragraphs describe the architectural components in detail.

\textbf{Lagrangian Formulation.} The constrained optimization problem in \eqref{eq:final_formulation} can be reformulated by introducing the Lagrangian function $\mathcal{L}(\delta_\theta, \lambda)$. Let $\lambda \ge 0$ be the Lagrange multiplier associated with the trust-region constraint \citep{chen2025towards}:
\begin{equation} \label{eq:lagrangian}
\mathcal{L}(\delta_\theta, \lambda) = \mathbb{E}_{s \sim \mathcal{D}, a \sim \pi_\beta(\cdot|s)} \left[ Q_{\phi}(s, a + \delta_\theta(s, a)) \right] - \lambda \left( \mathbb{E}_{s, a} \left[ \frac{1}{2} \delta_\theta(s, a)^\top I(s, a) \delta_\theta(s, a) \right] - \epsilon \right),
\end{equation}
where $I(s, a) = \nabla_a \log \pi_\beta \nabla_a \log \pi_\beta^\top$ denotes the Fisher information matrix at the state-action pair $(s, a)$. The optimal parameters are found by solving the minimax problem: $\min_{\lambda \ge 0} \max_{\theta} \mathcal{L}(\theta, \lambda)$.

\textbf{Inner Problem (Primal Update).} For a fixed dual variable $\lambda$, the inner problem aims to update the policy parameters $\theta$ to maximize the Lagrangian. Unlike traditional KL-constrained methods that require complex approximations, our quadratic reformulation allows for stable and efficient optimization via automatic differentiation:
\begin{equation}
\theta \leftarrow \theta + \eta_\theta \nabla_\theta \mathbb{E}_{s, a} \left[ Q_{\phi}(s, a + \delta_\theta(s, a)) - \frac{\lambda}{2} \delta_\theta(s, a)^\top I(s, a) \delta_\theta(s, a) \right],
\end{equation}
where $\eta_\theta$ is the learning rate. By leveraging the reparameterization trick through the transport map $T_s(a) = a + \delta_\theta(s, a)$, the gradient flows directly from the critic and the Fisher-based penalty into the residual network.

\textbf{Dual Update.} The dual variable $\lambda$ is updated to enforce the trust-region constraint. Given the current residual policy $\delta_\theta$, $\lambda$ is updated by minimizing the dual objective:
\begin{equation}
\lambda \leftarrow \mathrm{ReLU} \left( 
    \lambda + \eta \, \mathbb{E}_{s,a} \left[ 
    \frac{1}{2} \delta_{\theta}(s,a)^\top I(s,a) \delta_{\theta}(s,a) - \epsilon 
    \right] 
\right)
\end{equation}
In practice, to ensure that $\lambda$ remains non-negative, we typically optimize $\log \lambda$. This dual update automatically adjusts the ``penalty strength'' of the Fisher metric, ensuring the policy refinement stays within the provable neighborhood of the behavioral policy support.

\textbf{Further Discussion of the Lagrangian~\eqref{eq:lagrangian}.}
The inner optimization problem in~\eqref{eq:lagrangian} admits an insightful
analytical structure. By applying a first-order Taylor expansion to the critic
\(Q_\phi(s,a+\delta_\theta)\) around the behavioral action \(a\), the objective
becomes locally quadratic in \(\delta\):
\begin{equation}
\max_{\delta_\theta}
\mathbb{E}_{s,a}
\left[
\nabla_a Q_\phi(s,a)^\top\delta_\theta(s,a)
-
\frac{\lambda}{2}
\delta_\theta(s,a)^\top I(s,a)\delta_\theta(s,a)
\right].
\end{equation}
Taking the derivative with respect to \(\delta\) gives the normal equation
\[
    \lambda I(s,a)\delta = \nabla_a Q_\phi(s,a).
\]
Since the local Fisher matrix \(I(s,a)\) may be positive semidefinite, the pseudo-inverse solution
\[
    \delta^*(s,a)
    \approx
    \frac{1}{\lambda}I(s,a)^\dagger\nabla_a Q_\phi(s,a)
\]
should be interpreted as the minimum-norm stationary solution when
\(\nabla_a Q_\phi(s,a)\in\mathrm{Range}(I(s,a))\). If this condition fails,
the undamped objective may be unbounded along the null space of \(I(s,a)\).
In practice, this motivates using a small damped metric
\[
    I_\rho(s,a)=I(s,a)+\rho I_d,
    \qquad \rho>0,
\]
which yields the well-defined update
\begin{equation} \label{eq:closed_form_delta}
    \delta^*_\rho(s,a)
    \approx
    \frac{1}{\lambda}I_\rho(s,a)^{-1}\nabla_a Q_\phi(s,a).
\end{equation}
This expression reveals that the optimal policy refinement is effectively a pointwise Natural Gradient update in the action space, where the gradient of the critic is preconditioned by the inverse of local Fisher information matrix.

This closed-form solution suggests a significant opportunity for accelerating offline RL training. Traditionally, policy refinement requires iterative gradient-based optimization (as shown in the Primal Update). However,~\eqref{eq:closed_form_delta} implies that if the score function (and thus $I(s, a)$) is efficiently estimated, one could potentially bypass the inner loops of policy training entirely. By directly computing $\delta^*$ and updating the policy via a single supervised learning step, or even performing "on-the-fly" refinement during inference, the computational bottleneck of policy-critic coordination could be drastically reduced. We believe that leveraging this metric-aware analytical solution to design training-free or one-step refinement paradigms is a promising frontier for real-time offline RL applications.

\subsection{Deviated Optimal Solution Gap under Isotropic Metric.} \label{append:optimal_policy_gap}

We analyze the suboptimality induced by replacing the anisotropic metric $I(s,a)$ (or practical computation with $I_{\rho}(s,a)$ for one-step computation) with an isotropic identity matrix. Under the quadratic approximation in~\eqref{eq:lagrangian}, the objective takes the general linear form with the quadratic constraint:
\begin{equation}
\mathcal{L}(\delta_\theta,\lambda) = g^\top \delta - \frac{\lambda}{2} \delta^\top M \delta,
\end{equation}
where $g = \nabla_a Q_\phi(s,a)$ and $M = I_\rho(s,a)$.

The optimal solution under the anisotropic metric is:
\begin{equation}
\delta^*_{\text{ani}} = \frac{1}{\lambda} M^{^{\dagger}} g,
\end{equation}
yielding the optimal value:
\begin{equation}
\mathcal{L}^*_{\text{ani}} = \frac{1}{2\lambda} g^\top M^{^{\dagger}} g.
\end{equation}

In contrast, replacing $M$ with the identity matrix leads to the isotropic solution:
\begin{equation}
\delta^*_{\text{iso}} = \frac{1}{\lambda} g,
\end{equation}
with objective value:
\begin{equation}
\mathcal{L}^*_{\text{iso}} = \frac{1}{2\lambda} g^\top g.
\end{equation}

We define the optimality gap as:
\begin{equation}
\Delta(s,a) = \mathcal{L}^*_{\text{ani}} - \mathcal{L}^*_{\text{iso}} 
= \frac{1}{2\lambda} \left( g^\top M^{^{\dagger}} g - g^\top g \right).
\end{equation}

To further characterize this gap, consider the eigendecomposition $M = U \Lambda U^\top$, where $\Lambda = \mathrm{diag}(\lambda_i)$ and $U = [u_1, \cdots, u_d]$ ($\{u_i\}_{i = 1}^d$ is the orthogonal basis). Then the gap can be written as:
\begin{equation}
\Delta(s,a) = \frac{1}{2\lambda} \sum_i (u_i^\top g)^2 \left( \lambda_i^{-1} - 1 \right).
\end{equation}
This expression reveals that the suboptimality arises from ignoring the local curvature encoded in $M = I_{\rho}(s,a)$. Specifically, directions with small eigenvalues ($\lambda_i < 1$) are under-emphasized by the isotropic update, while directions with large eigenvalues ($\lambda_i > 1$) are over-amplified. As a result, the isotropic solution fails to properly balance exploration across directions, leading to degraded objective improvement.

In the special case where $M = I_d$, the gap vanishes. However, in general offline RL settings where the data distribution induces highly anisotropic geometry, this gap can be significant. This highlights the importance of anisotropic metric-aware updates for achieving optimal policy refinement.

\clearpage
\section{Implementation} \label{append:implementation}

\begin{algorithm}[ht] 
\caption{FiDec} 
\label{algo:FiDec}
\SetKwProg{Fn}{Function}{}{end}
\SetKwFunction{FMu}{$\mu_\beta$}
\SetKwFunction{FInfo}{$I$}
\SetKwFunction{FSample}{$T_s$}
\SetKwFunction{FSamplenext}{$T_{s^\prime}$}

\Fn{\FInfo{$s, a$}}{
    \Return $\frac{\big( t_\varepsilon v_{\beta}(t_\varepsilon, s, a) - a \big) \big( t_\varepsilon v_{\beta}(t_\varepsilon, s, a) - a \big)^{\top}}{(1 - t_\varepsilon)^2}$ \tcp{\textcolor{cyan}{Fisher information matrix}}
}
\medskip

\Fn{\FMu{$s, z$}}{
    \For{$t = 0$ \KwTo $M-1$}{
        $z \leftarrow z + \frac{1}{M} v_\beta\!\left(\tfrac{t}{M}, s, z\right)$
    }
    \Return $z$  \qquad  \qquad \qquad \qquad \ \  \tcp{\textcolor{cyan}{Behavioral policy}}
}
\medskip

\Fn{\FSample{$\mu_\beta(s, z)$}}{
    Sample $z \sim \mathcal{N}(0, I_d)$ \;
    \Return $a + \delta_\theta(s, \mu_\beta(s, z))$ \qquad \qquad \tcp{\textcolor{cyan}{Transport map}}
}
\medskip

\While{not converged}{
    Sample minibatch $\{(s, a, r, s')\} \sim \mathcal{D}$ \;
    
    \vspace{0.1em}
    \tcp{\textcolor{cyan}{Step (1): Update critic $Q_\phi$}}
    Sample $z \sim \mathcal{N}(0, I_d)$ \;
    $a' \leftarrow T_{s'}( \mu_\beta(s', z))$ \;
    Update $\phi$ by minimizing TD error:
    \[
        \mathbb{E}\Big[\big(Q_\phi(s, a) - r - \gamma \bar{Q}(s', a')\big)^2\Big]
    \]
    
    \vspace{0.1em}
    \tcp{\textcolor{cyan}{Step (2): Train vector field $v_\beta$}}
    Sample $x^0 \sim \mathcal{N}(0, I_d)$ \;
    $x^1 \leftarrow a$ \;
    Sample $t \sim \mathrm{Unif}[0, 1]$ \;
    $x^t \leftarrow (1 - t)x^0 + t x^1$ \;
    Update $\beta$ by minimizing:
    \[
         \mathbb{E}_{t \sim \text{Unit}[0,1], s, a = x^1  \sim \mathcal{D}, x^0 \sim  \mathcal{N}(0, I_d)}\Big[\|v_\beta(t, s, x^t) - (x^1 - x^0)\|_2^2\Big]
    \]
    
    \vspace{0.1em}
    \tcp{\textcolor{cyan}{Step (3): Train transport map $T$}}
    $\delta_\theta(s,a) \leftarrow T(s, \mu_\beta(s, z)) - \mu_\beta(s, z)$ \;
    Threshold $\epsilon$ \;
    Update $\theta$ by maximizing:
    \[
        \mathbb{E}_{s\sim \mathcal{D},\, z\sim \mathcal{N}(0, I_d)}\Big[
        Q_\phi(s, T_s( \mu_\beta(s, z))) 
        - \lambda \big( \frac{1}{2} \delta_{\theta}(s,a)^\top I(s,a)\delta_{\theta}(s,a) - \epsilon \big)
        \Big]
    \]
    
    \vspace{0.3em}
    \tcp{\textcolor{cyan}{Step (4): Update dual variable $\lambda$}}
    Step size $\eta$ \;
    \[
        \lambda \leftarrow \mathrm{ReLU}\!\big(
        \lambda + \eta \, \mathbb{E}_{s\sim \mathcal{D},\, a\sim\pi_\beta(\cdot|s))}\big[ \frac{1}{2}
        \delta_{\theta}(s,a)^\top I(s,a)\delta_{\theta}(s,a) - \epsilon
        \big]
        \big)
    \]
}
\Return transport map $T_s$
\end{algorithm}

\clearpage

\section{More Experimental Results.} \label{Append:more_results}

\subsection{Benchmarking Tasks}

\begin{figure}[ht]
    \centering
    \includegraphics[width=0.99\linewidth]{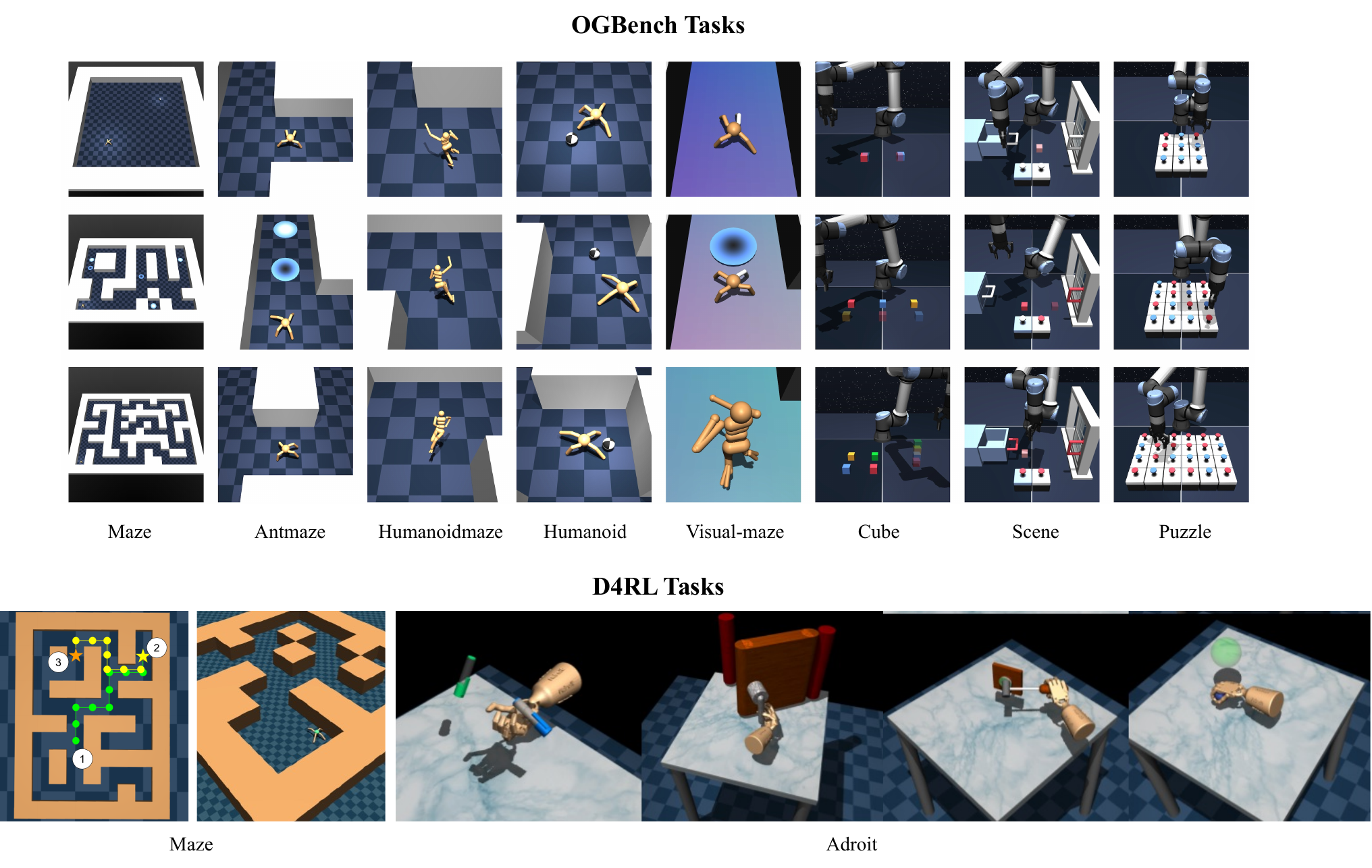}
    \caption{Overview of benchmark tasks. Our evaluation spans a diverse set of environments from OGBench \citep{park2024ogbench} (Maze, Antmaze, Humanoidmaze, Humanoid, Visual-maze, Cube, Scene, Puzzle) and D4RL \citep{fu2020d4rl} (Maze and Adroit), covering a wide range of challenges including navigation, locomotion, dexterous manipulation, and vision-based control.}
    \label{fig:benchmark}
\end{figure}

\subsection{Architectures and Hyperparameters}
Table~\ref{tab:fisher_hparam_summary} summarizes the main training configuration of our FiDec. We retain the same DeFlow/FQL-style backbone settings for optimization, network scale, critic update, flow integration, and evaluation protocol, while introducing the Fisher estimation time $t_{\varepsilon}$ as the key additional hyperparameter.

\begin{table}[ht]
\centering
\small
\setlength{\tabcolsep}{6pt}
\renewcommand{\arraystretch}{1.08}
\caption{Hyperparameters and configuration summary for the FiDec/DeFlow backbone.}
\label{tab:fisher_hparam_summary}
\begin{tabular}{l c | l c}
\toprule
\textbf{Hyperparameter} & \textbf{Value} & \textbf{Hyperparameter} & \textbf{Value} \\
\midrule
discount $\gamma$ (default) & 0.99 & target update rate $\tau$ & 0.005 \\
actor / critic learning rate & $3\times 10^{-4}$ & batch size & 256 \\
evaluation interval & $10^5$ & evaluation episodes & 50 \\
replay buffer size (minimum) & $2\times 10^6$ & save interval & $5\times 10^5$ \\
actor hidden dims & $(512,512,512,512)$ & critic hidden dims & $(512,512,512,512)$ \\
critic layer norm & True & actor layer norm & False \\
flow steps & 10 & flow time sampling & $\mathrm{Unif}[0,1]$ \\
double critic ensemble & 2 & gradient clipping & 5.0 \\
Q-normalization & True & use Lagrange tuning & True \\
BC / trust coefficient init. & 10.0 & fix BC flow online & False \\



\bottomrule
\end{tabular}

\vspace{2mm}
\par\noindent
\parbox{0.97\linewidth}{\footnotesize\textbf{Notes.} This table reports the shared backbone configuration together with the Fisher-specific settings used in the final runs. Following the latest wandb-confirmed runtime configuration, we use $t_{\varepsilon}=0.8$.
The replay buffer size is reported as a configured minimum, since the actual buffer allocation is implemented as $\max(\texttt{buffer\_size}, \texttt{dataset.size}+1)$.}
\end{table}

Table~\ref{tab:fisher_structure_summary} summarizes the model structure and computation flow. The BC flow policy, refinement head, and double-critic backbone remain aligned with the DeFlow/FQL-style implementation. The key component is a paradigm shift from an isotropic penalty to a local Fisher information matrix, estimated from the velocity field of flow. 

\begin{table*}[ht]
\centering
\small
\setlength{\tabcolsep}{5pt}
\renewcommand{\arraystretch}{1.08}
\caption{Architecture and computation summary of the FiDec.}
\label{tab:fisher_structure_summary}
\resizebox{1.\textwidth}{!}{%
\begin{tabular}{l l l c}
\toprule
\textbf{Module} & \textbf{Sub-Module} & \textbf{Structure / Computation} & \textbf{Output Dimension} \\
\midrule

\multirow{4}{*}{\makecell[c]{BC Flow Policy\\$v_\beta(s,x_t,t)$}}
& Input & Observation, corrected flow sample, and time scalar & $d_s + d_a + 1$ \\
& Backbone & MLP with hidden dims $(512,512,512,512)$ & 512 \\
& Update objective & Standard BC flow-matching loss & 1 \\
& Output & Velocity field in action space & $d_a$ \\

\midrule

\multirow{4}{*}{\makecell[c]{Transport map\\$T_s(a) = a +\delta_{\theta}(s,a)$}}
& Input & Observation and base action $a = \mu_\beta(s,a)$ & $d_s + d_a$ \\
& Backbone & MLP with hidden dims $(512,512,512,512)$ & 512 \\
& Role & Predict residual refinement $\delta_\theta(s,a)$ & $d_a$ \\
& Output action & $\mu_\beta(s,z) + \delta_\theta(s,a)$ & $d_a$ \\

\midrule

\multirow{5}{*}{\makecell[c]{Reward Critic\\$Q_\phi(s,a)$}}
& Input & Observation-action pair $(s,a)$ & $d_s + d_a$ \\
& Backbone & Value network with hidden dims $(512,512,512,512)$ & 512 \\
& Critic normalization & LayerNorm enabled & -- \\
& Ensemble & Double critic & 2 \\
& Output & Scalar $Q$ estimate & 1 \\

\midrule

\multirow{5}{*}{\makecell[c]{Fisher Metric\\Estimation}}
& Base action sampling & $a=\mathrm{ODESolve}(v_\beta,s,z)$, $z\sim\mathcal{N}(0,I)$ & $d_a$ \\
& Velocity query & Reuse BC flow network $v_\beta(s,x_t,t_{\varepsilon})$ & $d_a$ \\
& Score estimate & $\mathbf{s}_{t_{\varepsilon}} = (t_{\varepsilon}v_t - a) / (1-t_{\varepsilon})$ & $d_a$ \\
& Fisher penalty & $\frac{1}{2}\delta (s,a)^\top I(s,a) \delta (s,a)$ with trace-normalized $I(s,a)$ & 1 \\

\midrule

\multirow{3}{*}{\makecell[c]{Constraint\\Control}}
& Lagrange parameter & Learnable scalar $\log \lambda$ & 1 \\
& Trust-region control & Fisher-induced quadratic metric corresponding to KL divergence approximation & 1 \\
& Actor objective & BC flow loss + quadratic penalty + Q-maximization term & 1 \\

\bottomrule
\end{tabular}
}

\vspace{2mm}
\par\noindent
\parbox{0.97\linewidth}{\footnotesize\textbf{Notes.} The BC flow policy and refinement head share the same four-layer MLP scale as the DeFlow/FQL-style backbone. The main architectural difference is not a larger network, but the replacement of the isotropic residual penalty by a trace-normalized Fisher quadratic form estimated at $t_{\varepsilon}=0.8$.}
\end{table*}

\begin{table*}[ht]
\centering
\scriptsize
\setlength{\tabcolsep}{4pt}
\renewcommand{\arraystretch}{1.05}
\caption{Task-specific hyperparameters for offline RL. Prior columns follow the original task-specific settings reported in FQL and reused by DeFlow. We append our FiDec variant as the final column. DeFlow explicitly states that Q-normalization is enabled; we follow the same paper-level setting here. The only task-sensitive additional hyperparameter for the FiDec variant is the Fisher estimation time $t_{\varepsilon}$, fixed to $0.8$ in our final experiments.}
\label{tab:offline_hparams_full_with_fisher}

\begin{adjustbox}{max width=\textwidth}
\begin{tabular}{>{\raggedright\arraybackslash}p{5.5cm}ccccccccccc}
\toprule
Task(s)
& IQL
& ReBRAC
& IDQL
& SRPO
& CAC
& FAWAC
& FBRAC
& IFQL
& FQL
& DeFlow
& Fisher-Decorator \\
& $\alpha$
& $(\alpha_1,\alpha_2)$
& $N$
& $\beta$
& $\eta$
& $\alpha$
& $\alpha$
& $N$
& $\alpha$
& $\delta$
& $t_{\varepsilon}$ \\
\midrule

\makecell[l]{antmaze-large-navigate-singletask-\\task\{1,2,3,4,5\}-v0}
& 10 & (0.003, 0.01) & 32 & 0.3 & 1 & 3 & 3 & 32 & 10 & 0.01 & 0.8 \\

\makecell[l]{antmaze-giant-navigate-singletask-\\task\{1,2,3,4,5\}-v0}
& 10 & (0.003, 0.01) & 32 & 0.3 & 1 & 3 & 10 & 32 & 10 & 0.01 & 0.8 \\

\makecell[l]{humanoidmaze-medium-navigate-singletask-\\task\{1,2,3,4,5\}-v0}
& 10 & (0.01, 0.01) & 32 & 0.3 & 0.03 & 3 & 30 & 32 & 30 & 0.001 & 0.8 \\

\makecell[l]{humanoidmaze-large-navigate-singletask-\\task\{1,2,3,4,5\}-v0}
& 10 & (0.01, 0.01) & 32 & 0.3 & 1 & 3 & 30 & 32 & 30 & 0.001 & 0.8 \\

\makecell[l]{antsoccer-arena-navigate-singletask-\\task\{1,2,3,4,5\}-v0}
& 1 & (0.01, 0.01) & 32 & 0.03 & 1 & 10 & 30 & 64 & 10 & 0.01 & 0.8 \\

\makecell[l]{cube-single-play-singletask-\\task\{1,2,3,4,5\}-v0}
& 1 & (1, 0) & 32 & 0.03 & 0.003 & 1 & 100 & 32 & 300 & 0.001 & 0.8 \\

\makecell[l]{cube-double-play-singletask-\\task\{1,2,3,4,5\}-v0}
& 0.3 & (0.1, 0) & 32 & 0.1 & 0.3 & 0.3 & 100 & 32 & 300 & 0.001 & 0.8 \\

\makecell[l]{scene-play-singletask-\\task\{1,2,3,4,5\}-v0}
& 10 & (0.1, 0.01) & 32 & 0.1 & 0.3 & 0.3 & 100 & 32 & 300 & 0.001 & 0.8 \\

\makecell[l]{puzzle-3x3-play-singletask-\\task\{1,2,3,4,5\}-v0}
& 10 & (0.3, 0.01) & 32 & 0.1 & 0.01 & 0.3 & 100 & 32 & 1000 & 0.0005 & 0.8 \\

\makecell[l]{puzzle-4x4-play-singletask-\\task\{1,2,3,4,5\}-v0}
& 3 & (0.3, 0.01) & 32 & 0.1 & 0.01 & 0.3 & 300 & 32 & 1000 & 0.0005 & 0.8 \\

\makecell[l]{antmaze-umaze-v2\\antmaze-umaze-diverse-v2\\antmaze-medium-play-v2\\antmaze-medium-diverse-v2}
& - & - & - & - & 0.01 & 3 & 10 & 32 & 10 & 0.015 & 0.8 \\

\makecell[l]{antmaze-large-play-v2}
& - & - & - & - & 4.5 & 3 & 1 & 32 & 3 & 0.015 & 0.8 \\

\makecell[l]{antmaze-large-diverse-v2}
& - & - & - & - & 3.5 & 3 & 1 & 32 & 3 & 0.015 & 0.8 \\

\makecell[l]{pen-human-v1}
& - & - & 32 & 0.03 & 0.003 & 0.03 & 30000 & 32 & 10000 & 0.01 & 0.8 \\

\makecell[l]{pen-cloned-v1}
& - & - & 32 & 0.1 & 0.003 & 0.3 & 10000 & 32 & 10000 & 0.01 & 0.8 \\

\makecell[l]{pen-expert-v1}
& - & - & 32 & 0.1 & 0.03 & 0.1 & 30000 & 32 & 3000 & 0.01 & 0.8 \\

\makecell[l]{door-human-v1}
& - & - & 32 & 0.01 & 0.03 & 1 & 30000 & 32 & 30000 & 0.001 & 0.8 \\

\makecell[l]{door-cloned-v1}
& - & - & 32 & 0.03 & 0.03 & 1 & 10000 & 128 & 30000 & 0.001 & 0.8 \\

\makecell[l]{door-expert-v1}
& - & - & 32 & 0.01 & 0.03 & 3 & 30000 & 32 & 30000 & 0.001 & 0.8 \\

\makecell[l]{hammer-human-v1}
& - & - & 128 & 0.1 & 0.03 & 3 & 30000 & 32 & 30000 & 0.001 & 0.8 \\

\makecell[l]{hammer-cloned-v1}
& - & - & 32 & 0.1 & 0.003 & 0.03 & 10000 & 32 & 10000 & 0.001 & 0.8 \\

\makecell[l]{hammer-expert-v1}
& - & - & 32 & 0.03 & 0.03 & 3 & 30000 & 32 & 30000 & 0.001 & 0.8 \\

\makecell[l]{relocate-human-v1}
& - & - & 32 & 0.03 & 0.01 & 0.3 & 30000 & 128 & 10000 & 0.001 & 0.8 \\

\makecell[l]{relocate-cloned-v1}
& - & - & 64 & 0.03 & 0.01 & 0.1 & 3000 & 32 & 30000 & 0.001 & 0.8 \\

\makecell[l]{relocate-expert-v1}
& - & - & 32 & 0.01 & 0.003 & 1 & 30000 & 32 & 30000 & 0.001 & 0.8 \\

\makecell[l]{visual-cube-single-play-singletask-task1-v0}
& 1 & (1, 0) & - & - & - & - & 100 & 32 & 300 & 0.001 & 0.8 \\

\makecell[l]{visual-cube-double-play-singletask-task1-v0}
& 0.3 & (0.1, 0) & - & - & - & - & 100 & 32 & 100 & 0.001 & 0.8 \\

\makecell[l]{visual-scene-play-singletask-task1-v0}
& 10 & (0.1, 0.01) & - & - & - & - & 100 & 32 & 100 & 0.001 & 0.8 \\

\makecell[l]{visual-puzzle-3x3-play-singletask-task1-v0}
& 10 & (0.3, 0.01) & - & - & - & - & 100 & 32 & 300 & 0.0005 & 0.8 \\

\makecell[l]{visual-puzzle-4x4-play-singletask-task1-v0}
& 3 & (0.3, 0.01) & - & - & - & - & 300 & 32 & 300 & 0.0005 & 0.8 \\

\bottomrule
\end{tabular}
\end{adjustbox}

\vspace{2mm}
\parbox{0.97\linewidth}{\footnotesize
\textbf{Notes.}
We follow the task-specific hyperparameter conventions used in FQL and DeFlow, and append our algorithm (FiDec) as an additional column.
For the FiDec, the only task-sensitive additional hyperparameter is the Fisher estimation time $t_{\varepsilon}$, which is fixed to 0.8 across all offline tasks in our final experiments.
Other Fisher-specific constants are treated as fixed implementation details and are therefore omitted from the main comparison table. For the hyperparameter $\epsilon$, we keep it the same as in DeFlow for a fair comparison.}
\end{table*}

\clearpage
\subsection{Additional Results} \label{append:additional_result}

The additional examples in Figure~\ref{fig:new_examples} further illustrate that this advantage is not restricted to a single landscape geometry. In Figure~\ref{fig:new_examples}(a), the asymmetric bimodal landscape with an off-center saddle shows that isotropic regularization still induces either averaging across modes or interpolation through low-value regions, whereas FiDec concentrates refinement on the favorable high-value mode. Figure~\ref{fig:new_examples}(b) considers a curved support manifold with a detached hotspot: FQL is attracted toward the unsupported high-value region, DeFlow remains overly conservative on the support corridor, and FiDec advances farther along the valid manifold toward better actions. Figure~\ref{fig:new_examples}(c) presents a more complex multi-lobe landscape, where the central basin and eccentric valley make isotropic updates even more ambiguous; again, FiDec preserves the meaningful modal structure while steering mass toward high-value regions. Finally, Figure~\ref{fig:new_examples}(d) shows a crescent-like support geometry with multiple high-value tendencies, where FiDec consistently follows the data geometry and achieves support-preserving improvement, while FQL is prone to off-support attraction and DeFlow tends to under-exploit the available high-value directions. Taken together, these four examples highlight that the benefit of FiDec comes from geometry-aware anisotropic refinement rather than from any task-specific coincidence of the original two examples.

\begin{figure}[ht]
\centering

\includegraphics[width=0.98\linewidth]{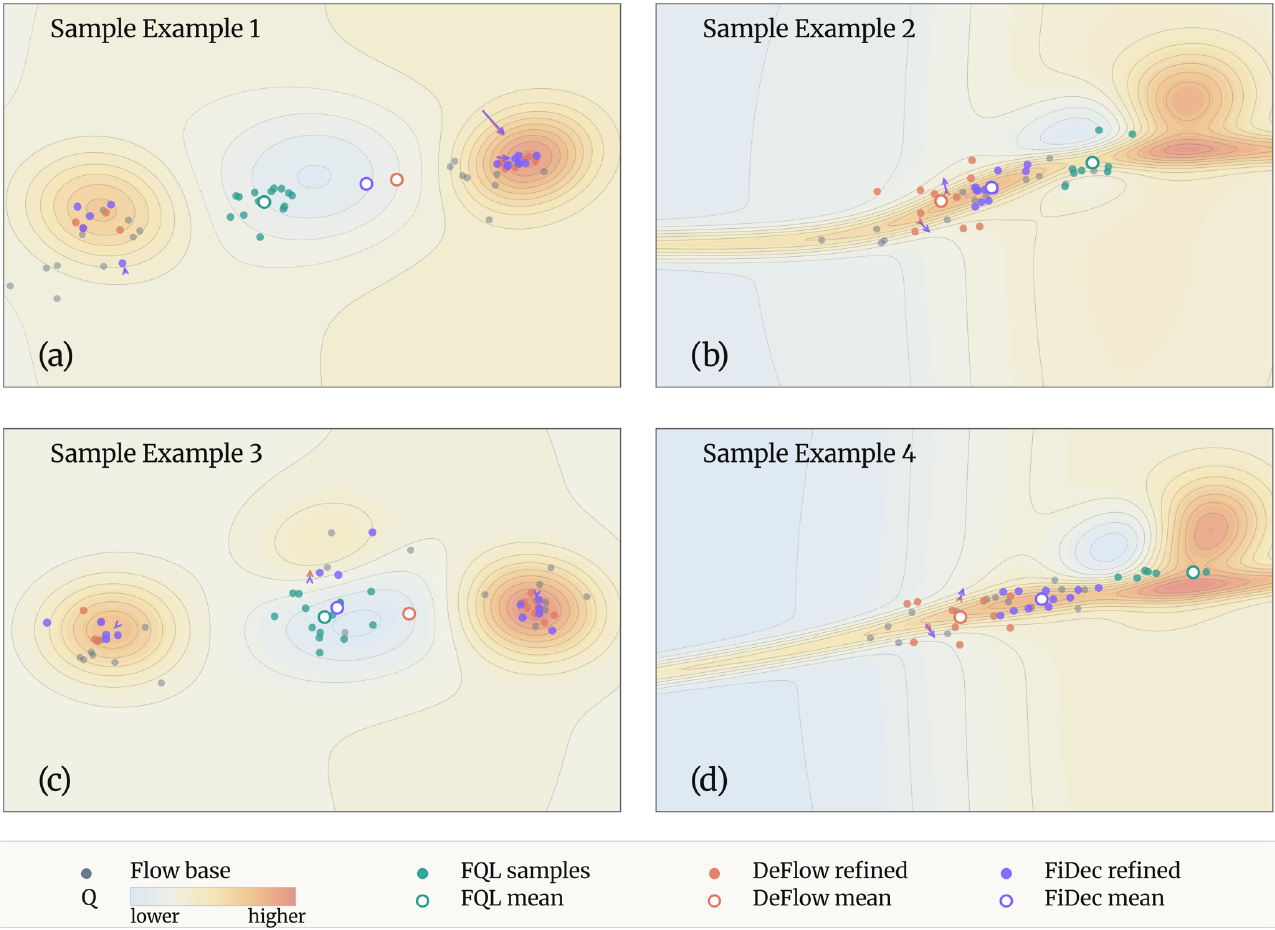}

\caption{Additional examples of isotropic and anisotropic policy refinement.
Panels (a) and (c) consider multimodal landscapes. Under isotropic regularization, FQL tends to average across modes, while DeFlow often interpolates through low-value regions or remains insufficiently biased toward favorable modes. In contrast, FiDec preserves multimodality while shifting mass toward higher-value regions by following the local geometry induced by the Fisher information matrix. 
Panels (b) and (d) consider policies supported on thin manifolds with detached or weakly supported high-value regions. In these cases, FQL is more easily attracted toward unsupported areas and exhibits out-of-distribution drift, whereas DeFlow remains comparatively conservative along the support and may fail to reach better regions. FiDec instead follows the support geometry and refines the policy toward higher-value valid regions. Together, these additional examples show that the advantage of FiDec is not tied to a single landscape, but persists across diverse multimodal and manifold-supported refinement scenarios.}
\label{fig:new_examples}
\vspace{-10pt}
\end{figure}

\begin{table*}[ht]
\centering
\caption{Full offline RL results corresponding to Table~\ref{tab:offline_rl_latest_final}. We report all 73 OGBench and D4RL offline tasks. (*) indicates the default task in each environment.}
\label{tab:full_offline_rl_latest_final}
\tiny
\setlength{\tabcolsep}{2.2pt}
\renewcommand{\arraystretch}{0.88}
\begin{adjustbox}{width=\textwidth}
\begin{tabular}{lcccccccccccc}
\toprule
& \multicolumn{3}{c}{Gaussian Policies} & \multicolumn{3}{c}{Diffusion Policies} & \multicolumn{6}{c}{Flow Policies} \\
\cmidrule(lr){2-4} \cmidrule(lr){5-7} \cmidrule(lr){8-13}
Task & BC & IQL & ReBRAC & IDQL & SRPO & CAC & FAWAC & FBRAC & IFQL & FQL & DeFlow & \textbf{FiDec} \\
\midrule
antmaze-large-navigate-singletask-task1-v0 (*) & 0 $\pm$0 & 48 $\pm$9 & \textbf{91 $\pm$10} & 0 $\pm$0 & 0 $\pm$0 & 42 $\pm$7 & 1 $\pm$1 & 70 $\pm$20 & 24 $\pm$17 & 80 $\pm$8 & \textbf{87$\pm$3} & \textbf{88$\pm$2} \\
antmaze-large-navigate-singletask-task2-v0 & 6 $\pm$3 & 42 $\pm$6 & \textbf{88 $\pm$4} & 14 $\pm$8 & 4 $\pm$4 & 1 $\pm$1 & 0 $\pm$1 & 35 $\pm$12 & 8 $\pm$3 & 57 $\pm$10 & 75$\pm$6 & 80$\pm$2 \\
antmaze-large-navigate-singletask-task3-v0 & 29 $\pm$5 & 72 $\pm$7 & 51 $\pm$18 & 26 $\pm$8 & 3 $\pm$2 & 49 $\pm$10 & 12 $\pm$4 & 83 $\pm$15 & 52 $\pm$17 & \textbf{93 $\pm$3} & \textbf{94$\pm$4} & \textbf{90$\pm$3} \\
antmaze-large-navigate-singletask-task4-v0 & 8 $\pm$3 & 51 $\pm$9 & \textbf{84 $\pm$7} & 62 $\pm$25 & 45 $\pm$19 & 17 $\pm$6 & 10 $\pm$3 & 37 $\pm$18 & 18 $\pm$8 & 80 $\pm$4 & 62$\pm$23 & \textbf{85$\pm$4} \\
antmaze-large-navigate-singletask-task5-v0 & 10 $\pm$3 & 54 $\pm$22 & \textbf{90 $\pm$2} & 2 $\pm$2 & 1 $\pm$1 & 55 $\pm$6 & 9 $\pm$5 & 76 $\pm$8 & 38 $\pm$18 & 83 $\pm$4 & \textbf{86$\pm$4} & \textbf{89$\pm$2} \\
\midrule
antmaze-giant-navigate-singletask-task1-v0 (*) & 0 $\pm$0 & 0 $\pm$0 & \textbf{27 $\pm$22} & 0 $\pm$0 & 0 $\pm$0 & 0 $\pm$0 & 0 $\pm$0 & 0 $\pm$1 & 0 $\pm$0 & 4 $\pm$5 & 6$\pm$2 & 15$\pm$1 \\
antmaze-giant-navigate-singletask-task2-v0 & 0 $\pm$0 & 1 $\pm$1 & 16 $\pm$17 & 0 $\pm$0 & 0 $\pm$0 & 0 $\pm$0 & 0 $\pm$0 & 4 $\pm$7 & 0 $\pm$0 & 9 $\pm$7 & 18$\pm$13 & \textbf{19$\pm$4} \\
antmaze-giant-navigate-singletask-task3-v0 & 0 $\pm$0 & 0 $\pm$0 & \textbf{34 $\pm$22} & 0 $\pm$0 & 0 $\pm$0 & 0 $\pm$0 & 0 $\pm$0 & 0 $\pm$0 & 0 $\pm$0 & 0 $\pm$1 & 0$\pm$0 & 0$\pm$0 \\
antmaze-giant-navigate-singletask-task4-v0 & 0 $\pm$0 & 0 $\pm$0 & 5 $\pm$12 & 0 $\pm$0 & 0 $\pm$0 & 0 $\pm$0 & 0 $\pm$0 & 9 $\pm$4 & 0 $\pm$0 & 14 $\pm$23 & 0$\pm$0 & \textbf{21$\pm$14} \\
antmaze-giant-navigate-singletask-task5-v0 & 1 $\pm$1 & 19 $\pm$7 & \textbf{49 $\pm$22} & 0 $\pm$1 & 0 $\pm$0 & 0 $\pm$0 & 0 $\pm$0 & 6 $\pm$10 & 13 $\pm$9 & 16 $\pm$28 & 0$\pm$0 & 10$\pm$12 \\
\midrule
humanoidmaze-medium-navigate-singletask-task1-v0 (*) & 1 $\pm$0 & 32 $\pm$7 & 16 $\pm$9 & 1 $\pm$1 & 0 $\pm$0 & 38 $\pm$19 & 6 $\pm$2 & 25 $\pm$8 & 69 $\pm$19 & 19 $\pm$12 & 35$\pm$4 & \textbf{80$\pm$5} \\
humanoidmaze-medium-navigate-singletask-task2-v0 & 1 $\pm$0 & 41 $\pm$9 & 18 $\pm$16 & 1 $\pm$1 & 1 $\pm$1 & 47 $\pm$35 & 40 $\pm$2 & 76 $\pm$10 & 85 $\pm$11 & \textbf{94 $\pm$3} & 76$\pm$2 & \textbf{92$\pm$1} \\
humanoidmaze-medium-navigate-singletask-task3-v0 & 6 $\pm$2 & 25 $\pm$5 & 36 $\pm$13 & 0 $\pm$1 & 2 $\pm$1 & \textbf{83 $\pm$18} & 19 $\pm$2 & 27 $\pm$11 & 49 $\pm$49 & 74 $\pm$18 & 75$\pm$3 & \textbf{84$\pm$7} \\
humanoidmaze-medium-navigate-singletask-task4-v0 & 0 $\pm$0 & 0 $\pm$1 & \textbf{15 $\pm$16} & 1 $\pm$1 & 1 $\pm$1 & 5 $\pm$4 & 1 $\pm$1 & 1 $\pm$2 & 1 $\pm$1 & 3 $\pm$4 & 11$\pm$2 & 5$\pm$1 \\
humanoidmaze-medium-navigate-singletask-task5-v0 & 2 $\pm$1 & 66 $\pm$4 & 24 $\pm$20 & 1 $\pm$1 & 3 $\pm$3 & 91 $\pm$5 & 31 $\pm$7 & 63 $\pm$9 & \textbf{98 $\pm$2} & \textbf{97 $\pm$2} & 88$\pm$4 & \textbf{97$\pm$1} \\
\midrule
humanoidmaze-large-navigate-singletask-task1-v0 (*) & 0 $\pm$0 & 3 $\pm$1 & 2 $\pm$1 & 0 $\pm$0 & 0 $\pm$0 & 1 $\pm$1 & 0 $\pm$0 & 0 $\pm$1 & 6 $\pm$2 & \textbf{7 $\pm$6} & 5$\pm$1 & 6$\pm$1 \\
humanoidmaze-large-navigate-singletask-task2-v0 & \textbf{0 $\pm$0} & \textbf{0 $\pm$0} & \textbf{0 $\pm$0} & \textbf{0 $\pm$0} & \textbf{0 $\pm$0} & \textbf{0 $\pm$0} & \textbf{0 $\pm$0} & \textbf{0 $\pm$0} & \textbf{0 $\pm$0} & \textbf{0 $\pm$0} & \textbf{0$\pm$0} & \textbf{0$\pm$0} \\
humanoidmaze-large-navigate-singletask-task3-v0 & 1 $\pm$1 & 7 $\pm$3 & 8 $\pm$4 & 3 $\pm$1 & 1 $\pm$1 & 2 $\pm$3 & 1 $\pm$1 & 10 $\pm$2 & \textbf{48 $\pm$10} & 11 $\pm$7 & 13$\pm$5 & 17$\pm$1 \\
humanoidmaze-large-navigate-singletask-task4-v0 & 1 $\pm$0 & 1 $\pm$0 & 1 $\pm$1 & 0 $\pm$0 & 0 $\pm$0 & 0 $\pm$1 & 0 $\pm$0 & 0 $\pm$0 & 1 $\pm$1 & 2 $\pm$3 & 5$\pm$1 & \textbf{7$\pm$4} \\
humanoidmaze-large-navigate-singletask-task5-v0 & 0 $\pm$1 & 1 $\pm$1 & 2 $\pm$2 & 0 $\pm$0 & 0 $\pm$0 & 0 $\pm$0 & 0 $\pm$0 & 1 $\pm$1 & 0 $\pm$0 & 1 $\pm$3 & 0$\pm$0 & \textbf{10$\pm$13} \\

\midrule

antsoccer-arena-navigate-singletask-task1-v0 & 2 $\pm$1 & 14 $\pm$5 & 0 $\pm$0 & 44 $\pm$12 & 2 $\pm$1 & 1 $\pm$3 & 22 $\pm$2 & 17 $\pm$3 & 61 $\pm$25 & 77 $\pm$4 & 84$\pm$4 & \textbf{89$\pm$4} \\
antsoccer-arena-navigate-singletask-task2-v0 & 2 $\pm$2 & 17 $\pm$7 & 0 $\pm$1 & 15 $\pm$12 & 3 $\pm$1 & 0 $\pm$0 & 8 $\pm$1 & 8 $\pm$2 & 75 $\pm$3 & 88 $\pm$3 & 87$\pm$2 & \textbf{94$\pm$2} \\
antsoccer-arena-navigate-singletask-task3-v0 & 0 $\pm$0 & 6 $\pm$4 & 0 $\pm$0 & 0 $\pm$0 & 0 $\pm$0 & 8 $\pm$19 & 11 $\pm$5 & 16 $\pm$3 & 14 $\pm$22 & \textbf{61 $\pm$6} & 56$\pm$5 & 56$\pm$2 \\
antsoccer-arena-navigate-singletask-task4-v0 (*) & 1 $\pm$0 & 3 $\pm$2 & 0 $\pm$0 & 0 $\pm$1 & 0 $\pm$0 & 0 $\pm$0 & 12 $\pm$3 & 24 $\pm$4 & 16 $\pm$9 & 39 $\pm$6 & 43$\pm$6 & \textbf{48$\pm$6} \\
antsoccer-arena-navigate-singletask-task5-v0 & 0 $\pm$0 & 2 $\pm$2 & 0 $\pm$0 & 0 $\pm$0 & 0 $\pm$0 & 0 $\pm$0 & 9 $\pm$2 & 15 $\pm$4 & 0 $\pm$1 & 36 $\pm$9 & \textbf{41$\pm$5} & 32$\pm$5 \\

\midrule 

cube-single-play-singletask-task1-v0 & 10 $\pm$5 & 88 $\pm$3 & 89 $\pm$5 & \textbf{95 $\pm$2} & 89 $\pm$7 & 77 $\pm$28 & 81 $\pm$9 & 73 $\pm$33 & 79 $\pm$4 & \textbf{97 $\pm$2} & 91$\pm$2 & \textbf{96$\pm$1} \\
cube-single-play-singletask-task2-v0 & 3 $\pm$1 & 85 $\pm$8 & 92 $\pm$4 & \textbf{96 $\pm$2} & 82 $\pm$16 & 80 $\pm$30 & 81 $\pm$9 & 83 $\pm$13 & 73 $\pm$3 & \textbf{97 $\pm$2} & \textbf{96$\pm$1} & \textbf{96$\pm$1} \\
cube-single-play-singletask-task3-v0 & 9 $\pm$3 & 91 $\pm$5 & 93 $\pm$3 & \textbf{99 $\pm$1} & \textbf{96 $\pm$2} & \textbf{98 $\pm$1} & 87 $\pm$4 & 82 $\pm$12 & 88 $\pm$4 & \textbf{98 $\pm$2} & \textbf{97$\pm$1} & \textbf{98$\pm$1} \\
cube-single-play-singletask-task4-v0 & 2 $\pm$1 & 73 $\pm$6 & \textbf{92 $\pm$3} & \textbf{93 $\pm$4} & 70 $\pm$18 & \textbf{91 $\pm$2} & 79 $\pm$6 & 79 $\pm$20 & 79 $\pm$6 & \textbf{94 $\pm$3} & 87$\pm$2 & 89$\pm$2 \\
cube-single-play-singletask-task5-v0 & 3 $\pm$3 & 78 $\pm$9 & 87 $\pm$8 & \textbf{90 $\pm$6} & 61 $\pm$12 & 80 $\pm$20 & 78 $\pm$10 & 76 $\pm$33 & 77 $\pm$7 & \textbf{93 $\pm$3} & 80$\pm$6 & \textbf{90$\pm$1} \\

\midrule

cube-double-play-singletask-task1-v0  (*) & 8 $\pm$3 & 27 $\pm$5 & 45 $\pm$6 & 39 $\pm$19 & 7 $\pm$6 & 21 $\pm$8 & 21 $\pm$7 & 47 $\pm$11 & 35 $\pm$9 & 61 $\pm$9 & \textbf{67$\pm$10} & \textbf{69$\pm$10} \\
cube-double-play-singletask-task2-v0 (*) & 0 $\pm$0 & 1 $\pm$1 & 7 $\pm$3 & 16 $\pm$10 & 0 $\pm$0 & 2 $\pm$2 & 2 $\pm$1 & 22 $\pm$12 & 9 $\pm$5 & 36 $\pm$6 & 53$\pm$3 & \textbf{58$\pm$9} \\
cube-double-play-singletask-task3-v0 & 0 $\pm$0 & 0 $\pm$0 & 4 $\pm$1 & 17 $\pm$8 & 0 $\pm$1 & 3 $\pm$1 & 1 $\pm$1 & 4 $\pm$2 & 8 $\pm$5 & 22 $\pm$5 & 40$\pm$8 & \textbf{51$\pm$13} \\
cube-double-play-singletask-task4-v0 & 0 $\pm$0 & 0 $\pm$0 & 1 $\pm$1 & 0 $\pm$1 & 0 $\pm$0 & 0 $\pm$1 & 0 $\pm$0 & 0 $\pm$1 & 1 $\pm$1 & 5 $\pm$2 & 3$\pm$2 & \textbf{6$\pm$1} \\
cube-double-play-singletask-task5-v0 & 0 $\pm$0 & 4 $\pm$3 & 4 $\pm$2 & 1 $\pm$1 & 0 $\pm$0 & 3 $\pm$2 & 2 $\pm$1 & 2 $\pm$2 & 17 $\pm$6 & 19 $\pm$10 & 28$\pm$3 & \textbf{32$\pm$7} \\

\midrule

scene-play-singletask-task1-v0 & 19 $\pm$6 & 94 $\pm$3 & \textbf{95 $\pm$2} & \textbf{100 $\pm$0} & 94 $\pm$4 & \textbf{100 $\pm$1} & 87 $\pm$8 & \textbf{96 $\pm$8} & \textbf{98 $\pm$3} & \textbf{100 $\pm$0} & \textbf{98$\pm$1} & \textbf{100$\pm$0} \\
scene-play-singletask-task2-v0 (*) & 1 $\pm$1 & 12 $\pm$3 & 50 $\pm$13 & 33 $\pm$14 & 2 $\pm$2 & 50 $\pm$40 & 18 $\pm$8 & 46 $\pm$10 & 0 $\pm$0 & \textbf{76 $\pm$9} & 59$\pm$5 & 72$\pm$5 \\
scene-play-singletask-task3-v0 & 1 $\pm$1 & 32 $\pm$7 & 55 $\pm$16 & \textbf{94 $\pm$4} & 4 $\pm$4 & 49 $\pm$16 & 38 $\pm$9 & 78 $\pm$14 & 54 $\pm$19 & \textbf{98 $\pm$1} & 88$\pm$4 & \textbf{96$\pm$1} \\
scene-play-singletask-task4-v0 & 2 $\pm$2 & 0 $\pm$1 & 3 $\pm$3 & 4 $\pm$3 & 0 $\pm$0 & 0 $\pm$0 & 6 $\pm$1 & 4 $\pm$4 & 0 $\pm$0 & 5 $\pm$1 & 8$\pm$3 & \textbf{12$\pm$3} \\
scene-play-singletask-task5-v0 & \textbf{0 $\pm$0} & \textbf{0 $\pm$0} & \textbf{0 $\pm$0} & \textbf{0 $\pm$0} & \textbf{0 $\pm$0} & \textbf{0 $\pm$0} & \textbf{0 $\pm$0} & \textbf{0 $\pm$0} & \textbf{0 $\pm$0} & \textbf{0 $\pm$0} & \textbf{0$\pm$0} & \textbf{0$\pm$0} \\

\midrule

puzzle-3x3-play-singletask-task1-v0 & 5 $\pm$2 & 33 $\pm$6 & \textbf{97 $\pm$4} & 52 $\pm$12 & 89 $\pm$5 & \textbf{97 $\pm$2} & 25 $\pm$9 & 63 $\pm$19 & 94 $\pm$3 & 90 $\pm$4 & \textbf{99$\pm$1} & \textbf{96$\pm$2} \\
puzzle-3x3-play-singletask-task2-v0 & 1 $\pm$1 & 4 $\pm$3 & 1 $\pm$1 & 0 $\pm$1 & 0 $\pm$1 & 0 $\pm$0 & 4 $\pm$2 & 2 $\pm$2 & 1 $\pm$2 & 16 $\pm$5 & 0$\pm$0 & \textbf{37$\pm$2} \\
puzzle-3x3-play-singletask-task3-v0 & 1 $\pm$1 & 3 $\pm$2 & 3 $\pm$1 & 0 $\pm$0 & 0 $\pm$0 & 0 $\pm$0 & 1 $\pm$0 & 1 $\pm$1 & 0 $\pm$0 & 10 $\pm$3 & 0$\pm$0 & \textbf{28$\pm$3} \\
puzzle-3x3-play-singletask-task4-v0 (*) & 1 $\pm$1 & 2 $\pm$1 & 2 $\pm$1 & 0 $\pm$0 & 0 $\pm$0 & 0 $\pm$0 & 1 $\pm$1 & 2 $\pm$2 & 0 $\pm$0 & 16 $\pm$5 & 3$\pm$3 & \textbf{24$\pm$1} \\
puzzle-3x3-play-singletask-task5-v0 & 1 $\pm$0 & 3 $\pm$2 & 5 $\pm$3 & 0 $\pm$0 & 0 $\pm$0 & 0 $\pm$0 & 1 $\pm$1 & 2 $\pm$2 & 0 $\pm$0 & 16 $\pm$3 & 20$\pm$10 & \textbf{30$\pm$2} \\

\midrule

puzzle-4x4-play-singletask-task1-v0 & 1 $\pm$1 & 12 $\pm$2 & 26 $\pm$4 & \textbf{48 $\pm$5} & 24 $\pm$9 & 44 $\pm$10 & 1 $\pm$2 & 32 $\pm$9 & \textbf{49 $\pm$9} & 34 $\pm$8 & 7$\pm$1 & 38$\pm$3 \\
puzzle-4x4-play-singletask-task2-v0 & 0 $\pm$0 & 7 $\pm$4 & 12 $\pm$4 & 14 $\pm$5 & 0 $\pm$1 & 0 $\pm$0 & 0 $\pm$1 & 5 $\pm$3 & 4 $\pm$4 & \textbf{16 $\pm$5} & 3$\pm$1 & 15$\pm$2 \\
puzzle-4x4-play-singletask-task3-v0 & 0 $\pm$0 & 9 $\pm$3 & 15 $\pm$3 & 34 $\pm$5 & 21 $\pm$10 & 29 $\pm$12 & 1 $\pm$1 & 20 $\pm$10 & \textbf{50 $\pm$14} & 18 $\pm$5 & 4$\pm$2 & 8$\pm$2 \\
puzzle-4x4-play-singletask-task4-v0 (*) & 0 $\pm$0 & 5 $\pm$2 & 10 $\pm$3 & \textbf{26 $\pm$6} & 7 $\pm$4 & 1 $\pm$1 & 0 $\pm$0 & 5 $\pm$1 & 21 $\pm$11 & 11 $\pm$3 & 4$\pm$1 & 6$\pm$2 \\
puzzle-4x4-play-singletask-task5-v0 & 0 $\pm$0 & 4 $\pm$1 & 7 $\pm$3 & \textbf{24 $\pm$11} & 1 $\pm$1 & 0 $\pm$0 & 0 $\pm$1 & 4 $\pm$3 & 2 $\pm$2 & 7 $\pm$3 & 2$\pm$1 & 9$\pm$1 \\

\midrule

antmaze-umaze-v2 & 55 & 77 & \textbf{98} & 94 & \textbf{97} & 66 $\pm$5 & 90 $\pm$6 & 94 $\pm$3 & 92 $\pm$6 & \textbf{96 $\pm$2} & \textbf{97$\pm$3} & \textbf{99$\pm$1} \\
antmaze-umaze-diverse-v2 & 47 & 54 & 84 & 80 & 82 & 66 $\pm$11 & 55 $\pm$7 & 82 $\pm$9 & 62 $\pm$12 & \textbf{89 $\pm$5} & 77$\pm$3 & 79$\pm$3 \\
antmaze-medium-play-v2 & 0 & 66 & \textbf{90} & 84 & 81 & 49 $\pm$24 & 52 $\pm$12 & 77 $\pm$7 & 56 $\pm$15 & 78 $\pm$7 & 78$\pm$7 & 83$\pm$2 \\
antmaze-medium-diverse-v2 & 1 & 74 & 84 & 85 & 75 & 0 $\pm$1 & 44 $\pm$15 & 77 $\pm$6 & 60 $\pm$25 & 71 $\pm$13 & 77$\pm$3 & \textbf{91$\pm$4} \\
antmaze-large-play-v2 & 0 & 42 & 52 & 64 & 54 & 0 $\pm$0 & 10 $\pm$6 & 32 $\pm$21 & 55 $\pm$9 & \textbf{84 $\pm$7} & 79$\pm$4 & \textbf{82$\pm$6} \\
antmaze-large-diverse-v2 & 0 & 30 & 64 & 68 & 54 & 0 $\pm$0 & 16 $\pm$10 & 20 $\pm$17 & 64 $\pm$8 & \textbf{83 $\pm$4} & 78$\pm$4 & \textbf{86$\pm$3} \\

\midrule

pen-human-v1 & 71 & 78 & \textbf{103} & 76 $\pm$10 & 69 $\pm$7 & 64 $\pm$8 & 67 $\pm$5 & 77 $\pm$7 & 71 $\pm$12 & 53 $\pm$6 & 65$\pm$7 & 41$\pm$6 \\
pen-cloned-v1 & 52 & 83 & \textbf{103} & 64 $\pm$7 & 61 $\pm$7 & 56 $\pm$10 & 62 $\pm$10 & 67 $\pm$9 & 80 $\pm$11 & 74 $\pm$11 & 72$\pm$6 & 94$\pm$12 \\
pen-expert-v1 & 110 & 128 & \textbf{152} & 140 $\pm$6 & 134 $\pm$4 & 103 $\pm$9 & 118 $\pm$6 & 119 $\pm$7 & 139 $\pm$5 & 142 $\pm$6 & 136$\pm$5 & 132$\pm$2 \\
door-human-v1 & 2 & 3 & -0 & 6 $\pm$2 & 3 $\pm$3 & 5 $\pm$2 & 2 $\pm$1 & 4 $\pm$2 & \textbf{7 $\pm$2} & 0 $\pm$0 & 0$\pm$0 & 0$\pm$0 \\
door-cloned-v1 & -0 & \textbf{3} & 0 & 0 $\pm$0 & 0 $\pm$0 & 1 $\pm$0 & 0 $\pm$1 & 0 $\pm$0 & 2 $\pm$2 & 2 $\pm$1 & 0$\pm$0 & 0$\pm$0 \\
relocate-cloned-v1 & -0 & 0 & \textbf{2} & -0 $\pm$0 & -0 $\pm$0 & -0 $\pm$0 & -0 $\pm$0 & 1 $\pm$1 & -0 $\pm$0 & -0 $\pm$0 & 0$\pm$0 & 0$\pm$0 \\
door-expert-v1 & \textbf{105} & \textbf{107} & \textbf{106} & \textbf{105 $\pm$1} & \textbf{105 $\pm$0} & 98 $\pm$3 & \textbf{103 $\pm$1} & \textbf{104 $\pm$1} & \textbf{104 $\pm$2} & \textbf{104 $\pm$1} & 32$\pm$17 & 88$\pm$2 \\
hammer-human-v1 & \textbf{3} & 2 & 0 & 2 $\pm$1 & 1 $\pm$1 & 2 $\pm$0 & 2 $\pm$1 & 2 $\pm$1 & \textbf{3 $\pm$1} & 1 $\pm$1 & 0$\pm$0 & 1$\pm$0 \\
hammer-cloned-v1 & 1 & 2 & 5 & 2 $\pm$1 & 2 $\pm$1 & 1 $\pm$1 & 1 $\pm$0 & 2 $\pm$1 & 2 $\pm$1 & \textbf{11 $\pm$9} & 6$\pm$3 & 7$\pm$2 \\
hammer-expert-v1 & 127 & \textbf{129} & \textbf{134} & 125 $\pm$4 & 127 $\pm$0 & 92 $\pm$11 & 118 $\pm$3 & 119 $\pm$9 & 117 $\pm$9 & 125 $\pm$3 & 29$\pm$12 & 92$\pm$12 \\
relocate-human-v1 & \textbf{0} & \textbf{0} & \textbf{0} & \textbf{0 $\pm$0} & \textbf{0 $\pm$0} & \textbf{0 $\pm$0} & \textbf{0 $\pm$0} & \textbf{0 $\pm$0} & \textbf{0 $\pm$0} & \textbf{0 $\pm$0} & \textbf{0$\pm$0} & \textbf{0$\pm$0} \\
relocate-expert-v1 & \textbf{108} & \textbf{106} & \textbf{108} & \textbf{107 $\pm$1} & \textbf{106 $\pm$2} & 93 $\pm$6 & \textbf{105 $\pm$3} & \textbf{105 $\pm$2} & \textbf{104 $\pm$3} & \textbf{107 $\pm$1} & 68$\pm$17 & 97$\pm$5 \\

\midrule

visual-cube-single-play-singletask-task1-v0$^{1}$ & - & 70 $\pm$12 & \textbf{83 $\pm$6} & - & - & - & - & 55 $\pm$8 & 49 $\pm$7 & 81 $\pm$12 & 71$\pm$7 & \textbf{86$\pm$2} \\
visual-cube-double-play-singletask-task1-v0$^{1}$ & - & \textbf{34 $\pm$23} & 4 $\pm$4 & - & - & - & - & 6 $\pm$2 & 8 $\pm$6 & 21 $\pm$11 & 7$\pm$4 & 15$\pm$7 \\
visual-scene-play-singletask-task1-v0$^{1}$ & - & \textbf{97 $\pm$2} & \textbf{98 $\pm$4} & - & - & - & - & 46 $\pm$4 & 86 $\pm$10 & \textbf{98 $\pm$3} & \textbf{99$\pm$1} & \textbf{98$\pm$1} \\
visual-puzzle-3x3-play-singletask-task1-v0$^{1}$ & - & 7 $\pm$15 & 88 $\pm$4 & - & - & - & - & 7 $\pm$2 & \textbf{100 $\pm$0} & 94 $\pm$1 & \textbf{97$\pm$1} & \textbf{96$\pm$1} \\
visual-puzzle-4x4-play-singletask-task1-v0$^{1}$ & - & 0 $\pm$0 & 26 $\pm$6 & - & - & - & - & 0 $\pm$0 & 8 $\pm$15 & \textbf{33 $\pm$6} & 22$\pm$7 & 23$\pm$1 \\

\bottomrule \multicolumn{13}{l}{\footnotesize $^{1}$ Due to the high computational cost of pixel-based tasks, only six strong methods are benchmarked on the visual OGBench tasks, following FQL/DeFlow.}
\end{tabular}
\end{adjustbox}
\end{table*}

\begin{table*}[ht]
\centering
\caption{Offline-to-online RL results. Results are averaged over 8 seeds.}
\label{tab:o2o_rl_latest_final}
\scriptsize
\setlength{\tabcolsep}{3.0pt}
\renewcommand{\arraystretch}{0.95}
\begin{adjustbox}{max width=\textwidth}
\begin{tabular}{>{\raggedright\arraybackslash}p{4.95cm}cccccccc}
\toprule
Task & IQL & ReBRAC & Cal-QL & RLPD & IFQL & FQL & DeFlow & \textbf{FiDec} \\
\midrule
humanoidmaze-medium-navigate-singletask-v0 & 21 $\pm$13$\!\rightarrow\!$16 $\pm$8 & 16 $\pm$20$\!\rightarrow\!$1 $\pm$1 & 0 $\pm$0$\!\rightarrow\!$0 $\pm$0 & 0 $\pm$0$\!\rightarrow\!$8 $\pm$10 & \colorbox{yellow!15}{56 $\pm$35$\!\rightarrow\!$82 $\pm$20} & 12 $\pm$7$\!\rightarrow\!$22 $\pm$12 & \colorbox{green!15}{36$\pm$7$\!\rightarrow\!$97$\pm$3} & \colorbox{cyan!15}{\textbf{82$\pm$4$\!\rightarrow\!$99$\pm$2}} \\

antsoccer-arena-navigate-singletask-v0 & 2 $\pm$1$\!\rightarrow\!$0 $\pm$0 & 0 $\pm$0$\!\rightarrow\!$0 $\pm$0 & 0 $\pm$0$\!\rightarrow\!$0 $\pm$0 & 0 $\pm$0$\!\rightarrow\!$0 $\pm$0 & 26 $\pm$15$\!\rightarrow\!$39 $\pm$10 & \colorbox{yellow!15}{28 $\pm$8$\!\rightarrow\!$86 $\pm$5} & \colorbox{green!15}{40$\pm$12$\!\rightarrow\!$87$\pm$1} & \colorbox{cyan!15}{\textbf{51$\pm$12$\!\rightarrow\!$91$\pm$4}} \\

cube-double-play-singletask-v0 & 0 $\pm$1$\!\rightarrow\!$0 $\pm$0 & 6 $\pm$5$\!\rightarrow\!$28 $\pm$28 & 0 $\pm$0$\!\rightarrow\!$0 $\pm$0 & 0 $\pm$0$\!\rightarrow\!$0 $\pm$0 & 12 $\pm$9$\!\rightarrow\!$40 $\pm$5 & \colorbox{cyan!15}{\textbf{40 $\pm$11$\!\rightarrow\!$92 $\pm$3}} & \colorbox{yellow!15}{49$\pm$6$\!\rightarrow\!$83$\pm$7} & \colorbox{green!15}{61$\pm$9$\!\rightarrow\!$87$\pm$2} \\

scene-play-singletask-v0 & 14 $\pm$11$\!\rightarrow\!$10 $\pm$9 & \textbf{55 $\pm$10$\!\rightarrow\!$100 $\pm$0} & 1 $\pm$2$\!\rightarrow\!$50 $\pm$53 & \textbf{0 $\pm$0$\!\rightarrow\!$100 $\pm$0} & 0 $\pm$1$\!\rightarrow\!$60 $\pm$39 & \colorbox{green!15}{82 $\pm$11$\!\rightarrow\!$100 $\pm$1} & \colorbox{yellow!15}{35$\pm$19$\!\rightarrow\!$100$\pm$0} & \colorbox{cyan!15}{\textbf{65$\pm$14$\!\rightarrow\!$100$\pm$0}} \\

puzzle-4x4-play-singletask-v0 & 5 $\pm$2$\!\rightarrow\!$1 $\pm$1 & 8 $\pm$4$\!\rightarrow\!$14 $\pm$35 & 0 $\pm$0$\!\rightarrow\!$0 $\pm$0 & \colorbox{yellow!15}{0 $\pm$0$\!\rightarrow\!$100 $\pm$1} & 23 $\pm$6$\!\rightarrow\!$19 $\pm$33 & 8 $\pm$3$\!\rightarrow\!$38 $\pm$52 & \colorbox{green!15}{3$\pm$1$\!\rightarrow\!$100$\pm$0} & \colorbox{cyan!15}{\textbf{7$\pm$3$\!\rightarrow\!$100$\pm$0}} \\

\midrule

antmaze-umaze-v2 & 77$\!\rightarrow\!$96 & 98$\!\rightarrow\!$75 & \colorbox{green!15}{77$\!\rightarrow\!$100} & 0 $\pm$0$\!\rightarrow\!$98 $\pm$3 & 94 $\pm$5$\!\rightarrow\!$96 $\pm$2 & \colorbox{green!15}{97 $\pm$2$\!\rightarrow\!$99 $\pm$1} & 88$\pm$2$\!\rightarrow\!$99$\pm$1 & \colorbox{cyan!15}{\textbf{97$\pm$1$\!\rightarrow\!$100$\pm$0}} \\

antmaze-umaze-diverse-v2 & 60$\!\rightarrow\!$64 & \textbf{74$\!\rightarrow\!$98} & 32$\!\rightarrow\!$98 & 0 $\pm$0$\!\rightarrow\!$94 $\pm$5 & 69 $\pm$20$\!\rightarrow\!$93 $\pm$5 & \colorbox{yellow!15}{79 $\pm$16$\!\rightarrow\!$100 $\pm$1} & \colorbox{green!15}{\textbf{79$\pm$9$\!\rightarrow\!$99$\pm$1}} & \colorbox{cyan!15}{\textbf{62$\pm$6$\!\rightarrow\!$100$\pm$0}} \\

antmaze-medium-play-v2 & 72$\!\rightarrow\!$90 & 88$\!\rightarrow\!$98 & \colorbox{cyan!15}{\textbf{72$\!\rightarrow\!$99}} & 0 $\pm$0$\!\rightarrow\!$98 $\pm$2 & 52 $\pm$19$\!\rightarrow\!$93 $\pm$2 & \colorbox{yellow!15}{77 $\pm$7$\!\rightarrow\!$97 $\pm$2} & 78$\pm$4$\!\rightarrow\!$98$\pm$0 & \colorbox{green!15}{83$\pm$2$\!\rightarrow\!$98$\pm$2} \\

antmaze-medium-diverse-v2 & 64$\!\rightarrow\!$92 & \colorbox{green!15}{85$\!\rightarrow\!$99} & \colorbox{yellow!15}{62$\!\rightarrow\!$98} & 0 $\pm$0$\!\rightarrow\!$97 $\pm$2 & 44 $\pm$26$\!\rightarrow\!$89 $\pm$4 & 55 $\pm$19$\!\rightarrow\!$97 $\pm$3 & 66$\pm$6$\!\rightarrow\!$97$\pm$2 & \colorbox{cyan!15}{\textbf{81$\pm$2$\!\rightarrow\!$98$\pm$2}} \\

antmaze-large-play-v2 & 38$\!\rightarrow\!$64 & 68$\!\rightarrow\!$32 & \colorbox{green!15}{32$\!\rightarrow\!$97} & 0 $\pm$0$\!\rightarrow\!$93 $\pm$5 & 64 $\pm$14$\!\rightarrow\!$80 $\pm$5 & 66 $\pm$40$\!\rightarrow\!$84 $\pm$30 & \textbf{67$\pm$10$\!\rightarrow\!$95$\pm$2} & \colorbox{cyan!15}{\textbf{71$\pm$10$\!\rightarrow\!$98$\pm$2}} \\

antmaze-large-diverse-v2 & 27$\!\rightarrow\!$64 & 67$\!\rightarrow\!$72 & 44$\!\rightarrow\!$92 & 0 $\pm$0$\!\rightarrow\!$94 $\pm$3 & 69 $\pm$6$\!\rightarrow\!$86 $\pm$5 & \colorbox{yellow!15}{75 $\pm$24$\!\rightarrow\!$94 $\pm$3} & \colorbox{green!15}{72$\pm$3$\!\rightarrow\!$94$\pm$2} & \colorbox{cyan!15}{\textbf{76$\pm$3$\!\rightarrow\!$95$\pm$2}} \\

\midrule

pen-cloned-v1 & 84$\!\rightarrow\!$102 & 74$\!\rightarrow\!$138 & -3$\!\rightarrow\!$-3 & 3 $\pm$2$\!\rightarrow\!$120 $\pm$10 & 77 $\pm$7$\!\rightarrow\!$107 $\pm$10 & \colorbox{cyan!15}{\textbf{53 $\pm$14$\!\rightarrow\!$149 $\pm$6}} & 67$\pm$7$\!\rightarrow\!$130$\pm$4 & \colorbox{green!15}{83$\pm$6$\!\rightarrow\!$143$\pm$7} \\

door-cloned-v1 & 1$\!\rightarrow\!$20 & \textbf{0$\!\rightarrow\!$102} & -0$\!\rightarrow\!$-0 &  \colorbox{yellow!15}{0$\pm$0$\!\rightarrow\!$102 $\pm$7} & 3 $\pm$2$\!\rightarrow\!$50 $\pm$15 & \colorbox{green!15}{0 $\pm$0$\!\rightarrow\!$102 $\pm$5} & 1$\pm$0$\!\rightarrow\!$99$\pm$1 & \colorbox{cyan!15}{\textbf{0$\pm$0$\!\rightarrow\!$104$\pm$2}} \\

hammer-cloned-v1 & 1$\!\rightarrow\!$57 & \textbf{7$\!\rightarrow\!$125} & 0$\!\rightarrow\!$0 & \colorbox{cyan!15}{\textbf{0 $\pm$0$\!\rightarrow\!$128 $\pm$29}} & 4 $\pm$2$\!\rightarrow\!$60 $\pm$14 & \colorbox{green!15}{0 $\pm$0$\!\rightarrow\!$127 $\pm$17} & 4$\pm$3$\!\rightarrow\!$120$\pm$1 & \colorbox{yellow!15}{7$\pm$3$\!\rightarrow\!$125$\pm$4} \\

\bottomrule
\end{tabular}
\end{adjustbox}
\end{table*}

\begin{table}[ht]
\centering
\caption{Ablation (full table for Figure~\ref{fig:perturbed_time}) on the perturbed time $t_\varepsilon$ over OGBench-5.
Results are reported as mean $\pm$ std over three runs for each task.}
\label{tab:perturbed_time_full}
\resizebox{\linewidth}{!}{
\begin{tabular}{lccccccc}
\toprule
Env & DeFlow (isotropic $L^2$ constraint) & $t_\varepsilon{=}0.70$ & $t_\varepsilon{=}0.75$ & $t_\varepsilon{=}0.80$ & $t_\varepsilon{=}0.85$ & $t_\varepsilon{=}0.90$ & $t_\varepsilon{=}0.95$ \\
\midrule
antsoccer-arena-navigate-singletask-task4      & 0.43$\pm$0.08 & 0.54$\pm$0.11 & 0.52$\pm$0.03 & 0.49$\pm$0.04 & 0.47$\pm$0.05 & 0.36$\pm$0.05 & 0.03$\pm$0.01 \\
cube-double-play-singletask-task2              & 0.50$\pm$0.03 & 0.57$\pm$0.13 & 0.51$\pm$0.03 & 0.58$\pm$0.02 & 0.47$\pm$0.06 & 0.45$\pm$0.04 & 0.62$\pm$0.06 \\
humanoidmaze-medium-navigate-singletask-task1  & 0.35$\pm$0.07 & 0.64$\pm$0.08 & 0.77$\pm$0.08 & 0.76$\pm$0.07 & 0.56$\pm$0.12 & 0.16$\pm$0.07 & 0.02$\pm$0.02 \\
puzzle-4x4-play-singletask-task4               & 0.04$\pm$0.01 & 0.03$\pm$0.02 & 0.03$\pm$0.00 & 0.05$\pm$0.01 & 0.07$\pm$0.02 & 0.09$\pm$0.01 & 0.04$\pm$0.02 \\
scene-play-singletask-task2                    & 0.51$\pm$0.06 & 0.61$\pm$0.01 & 0.62$\pm$0.09 & 0.70$\pm$0.07 & 0.71$\pm$0.06 & 0.73$\pm$0.01 & 0.77$\pm$0.04 \\
\midrule
\textbf{Avg} & \textbf{0.36$\pm$0.17} & \textbf{0.48$\pm$0.23} & \textbf{0.49$\pm$0.25} & \textbf{0.51$\pm$0.25} & \textbf{0.45$\pm$0.21} & \textbf{0.36$\pm$0.23} & \textbf{0.30$\pm$0.33} \\
\bottomrule
\end{tabular}}
\end{table}


\end{document}